\documentclass[twoside,11pt]{article}
\usepackage[preprint]{jmlr2e}
\usepackage{lastpage}
\jmlrheading{1}{2026}{1-\pageref{LastPage}}{5/26}{5/26}{26-0000}{Jian Qian and Shu Ge}
\ShortHeadings{$(\alpha,\beta)$-Stability for Boosting Vector-Valued Prediction}{Qian and Ge}
\firstpageno{1}

\title{\texorpdfstring{$(\alpha,\beta)$}{(alpha,beta)}-Stability for Boosting Vector-Valued Prediction}

\author{\name Jian Qian \email jianqian@hku.hk \\
       \addr The University of Hong Kong
       \AND
       \name Shu Ge \email geshu1026qq@gmail.com \\
        \addr Independent Researcher} 
\editor{}

\newcommand{\nonl}{\renewcommand{\nl}{\let\nl}}

\usepackage{aliascnt}

\makeatletter
\@ifundefined{lemma}{}{%

}
\makeatother

\newaliascnt{lemma}{theorem}
\newtheorem{lemma}[lemma]{Lemma}
\aliascntresetthe{lemma}

\newaliascnt{assumption}{theorem}

\aliascntresetthe{assumption}

\makeatletter
\@ifundefined{proposition}{}{%

}
\makeatother

\newaliascnt{proposition}{theorem}
\newtheorem{proposition}[proposition]{Proposition}
\aliascntresetthe{proposition}

\makeatletter
\@ifundefined{corollary}{}{%

}
\makeatother

\newaliascnt{corollary}{theorem}
\newtheorem{corollary}[corollary]{Corollary}
\aliascntresetthe{corollary}

\makeatletter
\@ifundefined{definition}{}{%

}
\makeatother

\newaliascnt{definition}{theorem}
\newtheorem{definition}[definition]{Definition}
\aliascntresetthe{definition}

\makeatletter
\@ifundefined{remark}{}{%

}
\makeatother

\newaliascnt{remark}{theorem}
\newtheorem{remark}[remark]{Remark}
\aliascntresetthe{remark}

\newcommand{\pfref}[1]{Proof of \Cref{#1}.}

\usepackage[utf8]{inputenc}
\usepackage{amsmath}
\usepackage{graphicx}
\usepackage{amsfonts}

\usepackage{mathtools}
\usepackage{mathrsfs}
\usepackage{xcolor}
\usepackage{booktabs}
\usepackage{bbold}
\usepackage{bm}
\usepackage{enumitem}
\usepackage{framed}
\usepackage{comment}
\usepackage{bbm}

\usepackage{algorithm}
\usepackage{algpseudocode}
\usepackage{cleveref}
\algrenewcommand\algorithmicrequire{\textbf{Input:}}

\crefname{lemma}{Lemma}{Lemmas}
\Crefname{lemma}{Lemma}{Lemmas}
\crefname{assumption}{Assumption}{Assumptions}
\Crefname{assumption}{Assumption}{Assumptions}
\crefname{proposition}{Proposition}{Propositions}
\Crefname{proposition}{Proposition}{Propositions}
\crefname{corollary}{Corollary}{Corollaries}
\Crefname{corollary}{Corollary}{Corollaries}
\crefname{definition}{Definition}{Definitions}
\Crefname{definition}{Definition}{Definitions}
\crefname{remark}{Remark}{Remarks}
\Crefname{remark}{Remark}{Remarks}

\makeatletter
\renewenvironment{proof}[1][Proof]{\par\noindent{\bfseries #1\ }}{\hfill\BlackBox\\[2mm]}
\makeatother

\newcommand{\bR}{\mathbb{R}}

\newcommand{\norm}[1]{\lVert {#1} \rVert}

\renewcommand{\leq}{\leqslant}
\renewcommand{\geq}{\geqslant}
\renewcommand{\le}{\leqslant}
\renewcommand{\ge}{\geqslant}
\newcommand{\argmin}{\mathop{\mathrm{arg}\,\mathrm{min}}}

\newcommand{\indic}{\mathbb{I}}

\newcommand{\eps}{\varepsilon}

\newcommand{\KL}{\mathrm{KL}}
\newcommand{\TV}{\mathrm{TV}}

\newcommand{\ldef}{\vcentcolon=}

\newcommand{\GKL}{\KL}
\newcommand{\Hels}{\Hel^2}

\newcommand{\cX}{\mathcal{X}}
\newcommand{\cY}{\mathcal{Y}}

\newcommand{\cH}{\mathcal{H}}

\usepackage{xspace}

\DeclarePairedDelimiter{\brk}{[}{]}

\DeclarePairedDelimiter{\prn}{(}{)}
\DeclarePairedDelimiter{\nrm}{\|}{\|}
\DeclarePairedDelimiter{\inner}{\langle}{\rangle}
\DeclarePairedDelimiter{\set}{\{}{\}}

\definecolor{myRed}{HTML}{D24D5C}
\definecolor{myGreen}{RGB}{46,139,87}

\usepackage{color-edits}
\addauthor{jq}{myRed}
\addauthor{sg}{myGreen}

\newcommand{\med}{\mathrm{med}}

\newcommand{\divv}{\mathrm{div}}
\newcommand{\Hel}{\mathrm{H}}
\newcommand{\R}{\mathbb{R}}
\newcommand{\convp}{\mathrm{co}_p}
\newcommand{\labelspace}[1][T]{\cY_{#1}}

\newcommand{\medboost}{\textsc{MedBoost}\xspace}
\newcommand{\adaboost}{\textsc{AdaBoost}\xspace}
\newcommand{\adaboostmh}{\textsc{AdaBoost.MH}\xspace}
\newcommand{\samme}{\textsc{SAMME}\xspace}
\newcommand{\geomedboost}{\textsc{GeoMedBoost}\xspace}
\newcommand{\WL}{\mathsf{WL}}

\newcommand{\cD}{\mathcal{D}}
\newcommand{\En}{\mathbb{E}}
\newcommand{\cS}{\mathcal{S}}
\newcommand{\sphere}{\ensuremath{\cS^{d-1}}}

\addtocontents{toc}{\protect\setcounter{tocdepth}{0}}

\begin{document}

\maketitle

\begin{abstract}
Despite the widespread use of boosting in structured prediction, a general theoretical understanding
of aggregation beyond scalar prediction remains incomplete.
We study vector-valued prediction under a target divergence and identify a
geometric stability property under which aggregation amplifies weak guarantees
into strong ones.
We formalize this property as $(\alpha,\beta)$-stability by geometric median and show how it
supports a boosting algorithm based on exponential reweighting and geometric-median aggregation.
For vector-valued prediction, we characterize this stability property under several natural
divergences: 
$\ell_1$ and $\ell_2$ distances for unconstrained vector-valued
prediction, and $\TV$, $\Hel$, and $\KL$ for density estimation over finite
probability vectors.
Building on these results, we propose a boosting algorithm \geomedboost.
Under a weak learner condition and $(\alpha,\beta)$-stability, we obtain exponential decay of the
empirical divergence error, which then yields population guarantees through a
generalization bound.
\end{abstract}


\section{Introduction}

Boosting is an ensemble method that turns a sequence of weak predictors into a
stronger predictor, typically by repeatedly fitting new weak predictors to the
examples or residuals emphasized by the current ensemble. It was originally
introduced as a method for improving performance in binary $0$--$1$
classification, most notably through AdaBoost \citep{freund1997decision}. It
was subsequently extended to real-valued regression
\citep{friedman2001greedy,kegl2003robust,mason1999functional,buhlmann2003boosting,chen2016xgboost}
and multi-class classification
\citep{schapire1997using,schapire1998improved,allwein2000reducing,zhu2009multi,mukherjeeschapire2013multiclass,brukhim2021multiclasscost,brukhim2023improper},
among many others. These developments led to a rich body of theory connecting
weak learnability, loss minimization, and stagewise additive modeling
\citep{mason1999functional,friedman2001greedy,buhlmann2003boosting,mukherjeeschapire2013multiclass,brukhim2021multiclasscost,brukhim2023improper}.
Beyond these settings, a number of works have also proposed boosting
algorithms for more general prediction problems, including structured outputs
\citep{ratliff2006mmpboost,shen2014structboost}.
A detailed discussion of both theoretical and algorithmic lines of work is deferred to \cref{sec:related-work}.

In this paper, we take a further step and ask whether one can give a comparable
theoretical understanding for general vector-valued prediction. Our starting
point is the median-boosting analysis of
\citet{kegl2003robust,kegl2004medianconvergence}, which replaces a real-valued
regression error by a binary exceedance surrogate: for a tolerance
$\varepsilon>0$, a prediction is counted as wrong when its absolute difference 
exceeds $\varepsilon$. We extend this exceedance viewpoint from scalar
deviations to general divergences on vector-valued outputs. 
Thus, we propose a boosting algorithm, \geomedboost, that obtains exponential decay of empirical exceedance loss and
then transfer this empirical guarantee to population error under mild assumptions through a
generalization argument.

The core conceptual contribution is the identification of the geometric
property that makes this boosting guarantee possible. We call this property
$(\alpha,\beta)$-stability by geometric median. It states that under the target divergence if more than an
$\alpha$ fraction of the weighted ensemble predictions lie within an
$\varepsilon$-ball around the target point, then the geometric-median aggregate lies
within a $\beta\varepsilon$-ball around the same target point. Thus a binary
exceedance guarantee for the ensemble at tolerance $\varepsilon$ is converted
into a binary exceedance guarantee for the final aggregate at the larger
tolerance $\beta\varepsilon$. We will term $\alpha$ the concentration level and $\beta$ the enlargement factor.

We show that this stability property holds for several natural divergences.
For unconstrained vector-valued prediction, geometric-median aggregation yields
sharp stability guarantees under the $\ell_1$ and $\ell_2$ distances. We then
turn to conditional density estimation, viewed as constrained vector-valued
prediction on the probability simplex. In this setting, total variation and
Hellinger distance mirror the behavior of $\ell_1$ and $\ell_2$, respectively.
Finally, for KL
divergence, direct geometric-median aggregation is not stable enough. Instead,
we obtain KL exceedance guarantees indirectly by aggregating under Hellinger
distance and converting the resulting Hellinger control to KL control.

All in all, we present a boosting algorithm for vector-valued
prediction, enabled by the geometric property of
$(\alpha,\beta)$-stability. The algorithm comes with theoretical guarantees for
empirical exceedance loss and population error, and the stability property is
verified for the natural divergences $\ell_1$, $\ell_2$, $\TV$, $\Hel$, and
$\KL$. These results advance the theoretical understanding of boosting for
vector-valued prediction and open the door to broader structured prediction
settings.

\paragraph{Contributions.}
Our main contributions are summarized as follows:
\begin{itemize}
  \item \textbf{Boosting algorithm and theoretical guarantees.}
  \begin{itemize}
    \item We propose \geomedboost, a boosting algorithm for
    vector-valued prediction under a target divergence. The algorithm combines
    exponential reweighting, weak learners measured through an
    $\varepsilon$-exceedance loss, and final aggregation by a geometric median. It recovers  classical algorithms such as \textsc{MedBoost}, \textsc{AdaBoost}, and \textsc{SAMME} as special cases. \looseness=-1
    \item Under the weak learner condition and
    $(\alpha,\beta)$-stability, we prove exponential decay of the empirical
    exceedance loss. We also show how this empirical guarantee transfers to
    population exceedance, and to expected divergence under boundedness assumption, via
    a generalization bound.
  \end{itemize}
  \item \textbf{\texorpdfstring{$(\alpha,\beta)$}{alpha,beta} stability through geometric median.}
  \begin{itemize}
    \item For unconstrained vector-valued prediction, we give sharp
    characterizations of $(\alpha,\beta)$-stability under the $\ell_1$ and
    $\ell_2$ distances. 
    The $\ell_1$ is $(1/2,d)$-stable where $d$ is the ambient dimension, while 
    the $\ell_2$ result gives a
    dimension-free tradeoff between the concentration level $\alpha$ and the
    enlargement factor $\beta$.
    \item For conditional density estimation over the probability simplex, we
    establish stability guarantees for total variation and Hellinger distance,
    paralleling the behavior of $\ell_1$ and $\ell_2$. For KL divergence, we
    show that direct geometric-median aggregation is not stable enough, but KL
    exceedance guarantees can be obtained indirectly through Hellinger
    aggregation.\looseness=-1
  \end{itemize}  
\end{itemize}

\paragraph{Organization.}
The remainder of the paper is organized as follows.
\Cref{sec:related-work} discusses connections to median boosting, multiclass boosting, and structured prediction by boosting.
\Cref{sec:prelim} introduces the prediction tasks and the weighted geometric median used for aggregation.
\Cref{sec:geomedboost-algorithm} presents the boosting algorithm
\geomedboost.
\Cref{sec:boostability} introduces $(\alpha,\beta)$-stability by geometric median and states the main
boosting guarantee built on this stability principle, together with the
generalization guarantees. \Cref{sec:stability-results} gives the stability
characterizations for unconstrained vector-valued prediction and conditional
density estimation.

\section{Related Work}
\label{sec:related-work}

\paragraph{Regression by median boosting.}
Median-type aggregation is a standard tool in robust estimation: aggregating multiple noisy candidates via a median
can strengthen weak concentration properties into high-probability guarantees \citep{lugosi2019subgaussianmean}.
More generally, \citet{minsker2015geomedian} study geometric median aggregation in normed spaces and provide geometric arguments establishing robustness guarantees for vector estimation.
In contrast to this robust-estimation setting, our use of median aggregation arises within a boosting algorithm and is limited to the final aggregation step.
Specifically, the median is used to combine a collection of weak predictors produced during the boosting procedure into a single output predictor, rather than serving as a standalone estimator.
For regression, \citet{kegl2003robust,kegl2004medianconvergence} analyze boosting algorithms in which averaging is replaced by a median-based aggregation rule.
In this paper, we use a closely related geometric argument to \citet{minsker2015geomedian} in the $\ell_2$ setting to analyze boosting for vector-valued prediction.
We further establish analogous geometric properties for other divergences, which characterize when median-type aggregation yields a valid final-step aggregation rule in boosting. \looseness=-1

\paragraph{Multiclass boosting and weak-learning criteria.}
Multiclass classification can be viewed as a finite-label instance of structured output prediction.
Boosting for multiclass classification has been studied with multiple, non-equivalent weak-learning formulations
\citep{mukherjeeschapire2013multiclass,brukhim2021multiclasscost,bressan2025dice}.
Representative algorithms include AdaBoost.MH/MR \citep{schapiresinger1999confidence} and direct multiclass variants such as
SAMME \citep{zhu2009multi}.
A substantial body of theoretical work analyzes the weak-learning conditions and performance
guarantees of multiclass boosting algorithms.
In particular, \citet{mukherjeeschapire2013multiclass} characterizes necessary-and-sufficient weak-learning conditions
for a tailored loss.
More recent work examines the resource/class dependence of multiclass boosting
and how to improve it via alternative weak-learning interfaces
\citep{brukhim2021multiclasscost,brukhim2022multiclasslearnability,charikar2022listlearnability,brukhim2023weakcriteria,brukhim2023improper}.
Relatedly, \citet{zou2024factorizedmh} gives a convergence guarantee for a factorized AdaBoost.MH variant.
In our setting, multiclass classification corresponds to a finite-dimensional vector-valued
prediction problem.
Our results provide a geometric perspective on aggregation that applies to this setting, recovering
certain existing guarantees as special cases while extending to more general structured output spaces.

\paragraph{Structured prediction by boosting.}
Structured prediction has been studied in settings such as vector-valued regression in RKHSs \citep{micchelli2005vector}.
In contrast, boosting for structured outputs is largely application-driven—e.g.,
column-generation/functional-gradient methods \citep{ratliff2006mmpboost,shen2014structboost} and
boosting potentials in structured models \citep{dietterich2004crfboost}—with guarantees that are
typically loss- and model-specific. This leaves a gap in general guarantees for learning structured
vector-valued predictors under constraints such as nonnegativity, normalization, or other geometric
restrictions on the output space.

\section{Preliminaries}
\label{sec:prelim}

Let $d\geq 1$ be an integer. We consider supervised learning with covariates $x \in \cX$ and vector-valued responses $y \in \cY \subseteq \R^d$. A predictor is a function $f : \cX \to \cY$, and performance is measured by a divergence $\divv : \cY \times \cY \to \R_{\geq 0}$ where $\divv(y,y)=0$ for all $y\in \cY$ and $\divv(y,y')>0$ for any $y\neq y'\in \cY$. Given a sample $\{(x_i,y_i)\}_{i=1}^n$ with $(x_i,y_i) \in \cX \times \cY$, we assume the data are drawn i.i.d.\ from an unknown distribution $\cD$. Our goal is to control the population risk
\[
\En_{(X,Y)\sim \cD}\!\left[\divv\bigl(Y, f(X)\bigr)\right].
\]

The main examples in this paper are the $\ell_1$ and $\ell_2$ distances for
general vector-valued prediction as well as conditional density
estimation as a constrained vector-valued prediction problem: in this case
$\cY$ is the probability simplex
$\Delta_d\;\ldef\;\set{p\in \R_+^d:\sum_{j=1}^d p_j=1}$
so each response is a probability mass function over a finite label space. On
the simplex, the divergences of interest are total variation $\TV$, Hellinger
distance $\Hel$, and KL divergence $\KL$. We defer the definitions of the divergences to the \cref{sec:stability-results}. 

Rather than directly minimizing the loss on the training samples, we use a boosting algorithm that iteratively improves performance via two components. First, a reweighting scheme generates a sequence of predictors, each satisfying a weak guarantee with respect to adaptively chosen sample weights. Second, these predictors are aggregated into a single predictor whose accuracy is measured under the same divergence $\divv$. This combination converts weak, weighted guarantees into a strong final guarantee.

The weak guarantee is formalized through a surrogate loss that upper bounds the $\varepsilon$-exceedance event induced by the divergence. Specifically, for any divergence $\divv$ and tolerance level $\varepsilon>0$, we introduce a loss function $C_\varepsilon$ satisfying
\[
C_\varepsilon(y,y') \;\ge\; \indic\{\divv(y,y')>\varepsilon\}
\qquad \forall\,y,y',
\]
where $\indic\{\cdot\}$ denotes the indicator of the event in
braces.
This loss formulation follows the path of \medboost\ \citep{kegl2003robust},
which extends the binary $0$--$1$ loss used in AdaBoost-style analyses
\citep{schapire1998improved} to absolute deviations exceedance between scalar
predictions. Here we take the next step by replacing scalar deviations with
divergence-based exceedance for vector-valued, and potentially more structured,
predictions.
Concretely, a generic choice is $C_\varepsilon(y,y') = \divv(y,y')/\varepsilon$, whereas for the one-dimensional $\ell_2$ divergence one may take the tighter quadratic surrogate $(y-y')^2/\varepsilon^2$. We use the following weak-learner interface.
\begin{definition}[Weak learner with parameter $\alpha$]
\label{def:weak-learner}
Let $0\leq \alpha\leq 1$. For any distribution $w$ over
$\{x_1,\dots,x_n\}$, a weak learner with parameter $\alpha$ is a map $\WL$
that outputs the hypothesis $h=\WL(w,C_\varepsilon)$ satisfying
\[
\sum_{i=1}^n w(x_i)\,C_{\varepsilon}\bigl(y_i,h(x_i)\bigr)\;\le\;1-\alpha.
\]
\end{definition}

If we take $C_{\varepsilon}(y,y') = \indic\{\divv(y,y')>\varepsilon\}$, the guarantee reduces to a weighted error condition: under any reweighting $w$, 
at least an $\alpha$ fraction (in weight) of the samples are predicted within $\varepsilon$ under $\divv$.

The aggregation rule must combine a collection of potentially noisy predictions into a single stable output. In our setting, this is achieved by extending the notion of the geometric median to general divergences over vector-valued outputs. Let $T\geq 1$ be an integer. Given weighted predictions $(\labelspace,\eta)$, the geometric median selects a point that minimizes the average divergence to the ensemble, thereby balancing the influence of all candidates \citep{weiszfeld1937point}.
Specifically, for a collection of predictions
$\labelspace = (y_1, \dots, y_T)$ with $y_t \in \cY$ and corresponding positive weights $\eta=(\eta_1,\dots,\eta_T) \in \Delta_T$, where we also write $\eta(y_t)=\eta_t$. We define the weighted geometric median as follows.
\begin{definition}[Weighted geometric median]
\label{def:geomed}
The \emph{weighted geometric median} of $(\labelspace,\eta)$ under a divergence $\divv$ is
\[
\med(\labelspace,\eta;\divv)\;\ldef\;\argmin_{g\in \cY}\sum_{t=1}^T \eta_t\,\divv(y_t,g).
\]
When $\divv(y,z)=\|y-z\|_p$, we write $\med_p(\labelspace,\eta)$. For $\divv\in\{\TV,\Hel,\sqrt{\KL}\}$ we write
$\med_{\TV}(\labelspace,\eta)$, $\med_{\Hel}(\labelspace,\eta)$, and
$\med_{\sqrt{\KL}}(\labelspace,\eta)$, respectively.
\end{definition}

For any set $B\subset \bR^d$, the associated notation $\eta(B)$ quantifies how much weight lies in a region $B$, $\eta(B)\ldef \sum_{t:y_t\in B} \eta_t$.
For any divergence $\divv$ and $\eps>0$, define the ball
\[
B_{\varepsilon}(z;\divv)\;\ldef\;\{y:\divv(z,y)\le \varepsilon\}
\]
to capture neighborhoods of accuracy around a reference point.

From a statistical viewpoint, the optimization problem above is an M-estimator
\citep{huber1964robust}. We use the term geometric median to emphasize the
geometric and stability aspects of the aggregation rule. This choice is also
computationally convenient: the divergences considered in this paper lead to
convex median objectives, which can be handled by standard convex-optimization
methods such as gradient descent. Understanding whether other aggregation rules
yield comparable stability guarantees is an interesting direction for future
work.

Together, these notions will allow us to formalize stability we desire in \cref{sec:boostability}.

\section{Geometric Median Boosting Algorithm}
\label{sec:geomedboost-algorithm}

\subsection{\texorpdfstring{$\gamma$}{gamma}-Perturbed Geometric Median}

We now describe how geometric-median aggregation enters the boosting algorithm.
The weighted geometric median introduced in \Cref{def:geomed} provides a
natural way to combine the round-wise predictions produced by boosting.
More generally, it is useful to consider an extension in which a $\gamma$
fraction of the total aggregation weight may be reassigned arbitrarily within
the ensemble. This perspective highlights the robustness of geometric-median
aggregation and will later recover the classical aggregation viewpoints behind the 
\adaboost\ \citep{bartlett1998boosting,ratsch2001marginal,ratsch2005efficient} and \medboost\
\citep{kegl2003robust}. To facilitate this discussion, we first introduce the
corresponding notion of perturbation.

A convenient way to formalize this is through a $\gamma$-neighborhood of
weighted sets. Passing from $(\mathcal Y,\eta)$ to a neighboring weighted set
allows up to a $\gamma$ portion of the total weight to be redistributed to any other points. 
The next definition makes this precise.

\begin{definition}[$\gamma$-neighborhood]
Let $(\labelspace,\eta)$ and $(\labelspace[T']',\eta')$ be a weighted set.
Define the distance between weighted sets by
\[
D\bigl((\labelspace,\eta),(\labelspace[T']',\eta')\bigr)
\;\ldef\;
\frac12\sum_{y\in \labelspace\cup\labelspace[T']'} |\eta(y)-\eta'(y)|,
\]
where $\eta(y)=0$ for $y\notin \labelspace$ and $\eta'(y)=0$ for $y\notin \labelspace[T']'$.
For $\gamma\in[0,1)$, define the $\gamma$-neighborhood
\[
\mathcal B_\gamma(\labelspace,\eta)
\;\ldef\;
\{(\labelspace[T']',\eta'):\ D((\labelspace,\eta),(\labelspace[T']',\eta'))\le\gamma\}.
\]
\end{definition}

\begin{definition}[$\gamma$-perturbed geometric median]
Let $0\le\gamma<1$.
The \emph{$\gamma$-perturbed (weighted) geometric median} of $(\labelspace,\eta)$
under divergence $\divv$ is
\[
\med(\labelspace,\eta;\divv,\gamma)
\;\ldef\;
\bigcup_{(\labelspace[T']',\eta')\in\mathcal B_\gamma(\labelspace,\eta)}
\med(\labelspace[T']',\eta';\divv).
\]
\end{definition}

The $\gamma$-perturbed geometric median therefore collects all outputs that can
arise after perturbing at most a $\gamma$ fraction of the weighted ensemble.
In the binary case $\cY=\set{0,1}$ with $0$--$1$ divergence and weights
$\eta=(\eta_0,\eta_1)$, the $\gamma$-perturbed geometric median as a set is equal to $\{0\}$ if
$\eta_0>\tfrac12+\gamma$, $\{1\}$ if $\eta_1>\tfrac12+\gamma$, and
$\{0,1\}$ otherwise. Thus a unique prediction exists exactly when one class has
vote margin exceeding $2\gamma$, matching the classical margin viewpoint in
\adaboost\ \citep{bartlett1998boosting}; an analogous interpretation is also true for median-based regression, matching \medboost \citep{kegl2003robust}.
In \geomedboost, this general notion is instantiated pointwise with
$\labelspace=\{h_t(x)\}_{t=1}^T$ and normalized aggregation weights
$\{\eta_t'\}_{t=1}^T$.

\subsection{The Algorithm \geomedboost}
\label{sec:weak-learner-and-geomedboost}

The algorithm \geomedboost (\cref{alg:boosting-geo-med}) follows the standard
boosting procedure, so the reweighting mechanism is intentionally close to
classical boosting. It maintains sample weights $w_t$, repeatedly invokes the
weak learner under the current weighted sample to obtain $h_t$, and then
updates the sample weights by exponential reweighting so that examples with
larger surrogate loss receive more emphasis in later rounds.

Each round also outputs a boosting coefficient $\eta_t$, which is later
normalized into the aggregation weight $\eta_t'$. This coefficient is chosen by
minimizing a one-dimensional convex objective \eqref{alg:def:eta}. When the weak learner improves
strictly over the target level $1-\alpha$, the minimizing coefficient is
positive, and this positivity ensures that the surrogate loss decreases across
rounds, as shown in the analysis below (\cref{sec:main-boosting-theorem}).\looseness=-1

What changes in our setting is the final aggregation step. Classical boosting
algorithms aggregate round-wise hypotheses by averaging, taking signs, or
voting over labels, depending on the prediction space. For vector-valued outputs, these scalar-specific aggregation rules are no longer
canonical. \geomedboost keeps the weak-to-strong reweighting logic of
classical boosting, but replaces the final scalar aggregation rule with a
$\gamma$-perturbed geometric median under $\divv$. After $T$ rounds, the final
predictor is obtained by aggregating $\{h_t(x)\}_{t=1}^T$ in this way using the
normalized weights $\{\eta_t'\}_{t=1}^T$. The perturbation parameter $\gamma$
is used to state a robustness guarantee: even if a $\gamma$ portion of the
weighted ensemble is reassigned arbitrarily, 
the same analysis still controls
the aggregated prediction adaptively to $\gamma$. \looseness=-1

\begin{algorithm}
\caption{\geomedboost}
\label{alg:boosting-geo-med}
\begin{algorithmic}[1]
\Require Data $\{(x_i,y_i)\}_{i=1}^n$, weak learner $\WL$, loss $C_\varepsilon$, divergence $\divv$, rounds $T$, concentration level $\alpha$, perturbation level $\gamma$.
\State $w_1(x_i)=1/n$ for all $i$.
\For{$t=1,\dots,T$}
\State $h_t\leftarrow \WL(w_t,C_\varepsilon)$.
\State Compute the round boosting coefficient $\eta_t$ by
\Statex
\begin{equation}
\eta_t=\argmin_{\eta>0}\sum_i w_t(x_i)\exp\bigl(-\eta(1-\alpha-C_\varepsilon(y_i,h_t(x_i)))\bigr).
\label{alg:def:eta}
\end{equation}
\State $\displaystyle
w_{t+1}(x_i)\propto w_t(x_i)\exp\bigl(\eta_t C_\varepsilon(y_i,h_t(x_i))\bigr)$.
\EndFor
\State $\eta_t'\leftarrow \eta_t/\sum_{s=1}^T\eta_s$ for $t\in [T]$.
\State Output $f_T(x)\in \med(\{h_t(x)\}_{t=1}^T,\set{\eta_t'}_{t=1}^T;\divv,\gamma)$.
\end{algorithmic}
\end{algorithm}

\paragraph{Connection to Classical Boosting Methods}
This formulation makes the relationship with classical boosting explicit. 
In binary classification with the
$0$--$1$ divergence, take $\varepsilon=0$ and
$C_\varepsilon(y_i,h_t(x_i))=\indic\{h_t(x_i)\ne y_i\}$. Then the objective in
\eqref{alg:def:eta} becomes
\[
\sum_{i:h_t(x_i)\ne y_i} w_t(x_i)e^{\eta/2}
\;+\;
\sum_{i:h_t(x_i)=y_i} w_t(x_i)e^{-\eta/2}.
\]
Writing the weighted error as
$\epsilon_t=\sum_i w_t(x_i)\indic\{h_t(x_i)\ne y_i\}$, the minimizer is
$\eta_t=\log((1-\epsilon_t)/\epsilon_t)$. The final geometric median is then
weighted majority vote.
This is precisely the AdaBoost mechanism in geometric form
\citep{freund1997decision}. In multi-class classification, the same viewpoint
yields weighted class voting, and combined with the exponential reweighting
step above it specializes to \samme \citep{zhu2009multi} and \adaboostmh
\citep{schapire1998improved}. In one-dimensional regression with absolute loss, the geometric median coincides
with the weighted median, so the final aggregation step reduces to the rule
used by \medboost \citep{kegl2003robust}. Thus, \geomedboost 
preserves the classical boosting algorithms in their corresponding tasks.

\section{\texorpdfstring{$(\alpha,\beta)$}{(alpha,beta)}-Stability and Boosting Guarantee}
\label{sec:boostability}
\label{sec:main-boosting-theorem}

We now introduce the core stability property of any given divergence that makes \geomedboost effective.
In general, one could call a divergence stable if there exists an aggregation rule
with the following property: whenever an $\alpha$ fraction of the weighted ensemble lies within
an $\eps$-ball around a target point $z$, every admissible aggregate lies within a
$\beta\eps$-ball around the same target. The aggregation rule is therefore allowed to enlarge
the accuracy scale by a controlled factor $\beta$, but it is not allowed to move away from a
region that already contains enough ensemble weight.

In this paper we focus on the geometric median as the aggregation rule. This is the most natural
choice for high-dimensional vector-valued predictions: it is defined directly from the divergence,
does not rely on scalar ordering, and reduces to familiar median or voting rules in classical
special cases.

\begin{definition}[$(\alpha,\beta)$-stability by geometric median]
Let $\alpha\in (\tfrac12,1)$ and $\beta\geq 1$. We say that a divergence $\divv$ is
\emph{$(\alpha,\beta)$-stable by geometric median} if geometric-median aggregation has the
concentration-preservation property described above. Concretely, for every weighted ensemble
$(\labelspace,\eta)$, every target point $z$, and every accuracy level $\eps\geq 0$, if at least
an $\alpha$ fraction of the ensemble weight lies in the $\divv$-ball
$B_{\varepsilon}(z;\divv)$, then every geometric median of the ensemble lies in the enlarged ball
$B_{\beta\varepsilon}(z;\divv)$:
\[
\eta\!\left(B_{\varepsilon}(z;\divv)\right)\geq\alpha
\quad\Longrightarrow\quad
\med(\labelspace,\eta;\divv)\subseteq B_{\beta\varepsilon}(z;\divv).
\]
We refer to $\alpha$ as the concentration level and $\beta$ as the enlargement factor.
\end{definition}

This stability condition is the main conceptual contribution of the paper. It isolates the
geometric fact needed for boosting: once enough predictions concentrate around a point,
geometric-median aggregation converts that weighted concentration into accuracy of the final
prediction. As a simple consequence, if a single prediction carries weight exceeding $\alpha$,
then taking $\eps=0$ forces every geometric median to coincide with that prediction.

The definition is stated for arbitrary weighted ensembles rather than only for uniform distributions on two (or any fixed number of) points. This generality is likely necessary: the lower-bound constructions for
$\ell_1$ and $\ell_2$ stability in \cref{sec:boostability-for-vec-prediction} requires weighted ensemble with the number of points scaling with the dimension and is unclear how to simplify in terms of the weights or the number of points.

For a predictor $f$ and radius $r\geq 0$, define the empirical exceedance loss
\[
L_{\divv,r}(f)
\;\ldef\;
\frac1n\sum_{i=1}^n \indic\{\divv(y_i,f(x_i))>r\}.
\]
\begin{theorem}
\label{thm:error-bound}
Suppose $\WL$ is a weak learner with parameter $\alpha$, $\divv$ is
$(\alpha-\gamma,\beta)$-stable by geometric median, and $f_T$ is the output of
\geomedboost. Then
\[
L_{\divv,\beta\varepsilon}(f_T)
\;\le\;
\prod_{t=1}^{T}
\sum_{i=1}^n
w_t(x_i)\exp\!\Bigl(-\eta_t\bigl(1-\alpha-C_{\varepsilon}(y_i,h_t(x_i))\bigr)\Bigr).
\]
Moreover, if for some $\zeta>0$ with $\alpha+\zeta\leq 1$ the same weak learner
has uniform edge $\zeta$, meaning it is a weak learner with parameter
$\alpha+\zeta$, then, with
\[
M\;\ldef\;
\sup_{t,i} C_\varepsilon(y_i,h_t(x_i))
-
\inf_{t,i} C_\varepsilon(y_i,h_t(x_i)),
\]
we have
\[
L_{\divv,\beta\varepsilon}(f_T)
\leq
\exp\!\left(-\frac{2\zeta^2}{M^2}T\right).
\]
Consequently, the empirical exceedance loss becomes zero once
$T>M^2\log(n)/(2\zeta^2)$.
\end{theorem}

The iteration bound is especially transparent when the range parameter $M$ is
known in advance. For example, if $C_\varepsilon$ is an indicator loss, then
$M\leq 1$, so the number of rounds needed to drive the empirical exceedance
loss to zero can be fixed before running the algorithm.

The theorem has the same multiplicative form as classical boosting guarantees. At round $t$, the boosting
coefficient $\eta_t$ is positive 
under strict weak improvement.
Thus each successful
round contributes a factor smaller than one to the empirical exceedance bound. 
The parameter
$\gamma$ makes the stability requirement stronger since the final aggregation must remain controlled
even after a $\gamma$ fraction of the aggregation weight is reassigned arbitrarily. With the
appropriate specialization of the divergence and aggregation rule, this recovers the usual
AdaBoost-type guarantees \citep{schapire1998improved}, the margin-based AdaBoost guarantees
\citep{bartlett1998boosting,ratsch2001marginal,ratsch2005efficient}, and the median-aggregation
guarantee of \medboost\ \citep{kegl2003robust}.

The proof is a direct generalization of the \medboost\ argument. The exponential reweighting
analysis is essentially unchanged; the new ingredient is the $(\alpha,\beta)$-stability
implication, which turns weighted concentration of the weak predictions into an exceedance bound
for the aggregated predictor. The next section verifies this stability property for the
divergences considered in the paper.

The theorem is an empirical statement about divergences on the training sample. Population and
generalization guarantees are deferred to \cref{sec:generalization}.

We first record the key consequence of stability needed for the exponential-potential argument.

\begin{lemma}
\label{lem:core}
Suppose $\WL$ is a weak
learner with parameter $\alpha$ and the divergence $\divv$ is $(\alpha-\gamma,\beta)$-stable by geometric median. If
$\divv(y_i, f_T(x_i)) > \beta\varepsilon$, then
\[
\sum_{t=1}^{T} \eta_t (1-\alpha - C_{\varepsilon}(y_i,h_t(x_i))) \leq 0.
\]
\end{lemma}

\begin{proof}[\pfref{lem:core}]
By the definition of the $\gamma$-perturbed median, there exists a perturbed weighted ensemble
$(\set{u_s}_{s=1}^S,\set{\rho_s}_{s=1}^S)\in
\mathcal B_\gamma(\{h_t(x_i)\}_{t=1}^T,\set{\eta_t'}_{t=1}^T)$ such that
\[
f_T(x_i)\in
\med(\set{u_s}_{s=1}^S,\set{\rho_s}_{s=1}^S;\divv).
\]
The definition of $(\alpha-\gamma,\beta)$-stability and the assumption
$\divv(y_i,f_T(x_i))>\beta\varepsilon$ imply
\[
\sum_{s=1}^S \rho_s \indic\{\divv(y_i,u_s)\leq \varepsilon\}\leq \alpha-\gamma.
\]
Since the perturbed ensemble differs from the original ensemble by at most $\gamma$ reassigned
mass, the original normalized aggregation weights satisfy
\[
\sum_{t=1}^T \eta_t' \indic\{\divv(y_i,h_t(x_i))\leq \varepsilon\}\leq \alpha.
\]
Equivalently, at least a $(1-\alpha)$ fraction of the unnormalized aggregation weight lies outside
the $\varepsilon$-ball:
\[
\sum_{t=1}^{T} \eta_t \indic\{\divv(y_i,h_t(x_i))> \varepsilon\}
\geq
(1-\alpha)\sum_{t=1}^{T}\eta_t.
\]
Using $C_\varepsilon(y,y')\geq \indic\{\divv(y,y')>\varepsilon\}$, we get
\[
\sum_{t=1}^{T} \eta_t C_{\varepsilon}(y_i,h_t(x_i))
\geq
(1-\alpha)\sum_{t=1}^{T}\eta_t.
\]
Rearranging gives the claim.
\end{proof}

\begin{proof}[\pfref{thm:error-bound}]
By \cref{lem:core},
\begin{align*}
L_{\divv,\beta\varepsilon}(f_T)
&= \frac{1}{n} \sum_{i=1}^{n} \indic \{ \divv(y_i, f_T(x_i)) >\beta \varepsilon \} \\
&\leq \frac{1}{n} \sum_{i=1}^{n}
\indic \left\{ -\sum_{t=1}^{T} \eta_t
\prn*{1-\alpha - C_{\varepsilon}(y_i,h_t(x_i))} \geq 0 \right\} \\
&\leq \frac{1}{n} \sum_{i=1}^{n}
\exp \left( -\sum_{t=1}^{T} \eta_t
\prn*{1-\alpha - C_{\varepsilon}(y_i,h_t(x_i))} \right) \\
&= \prod_{t=1}^{T}
\left( \sum_{i=1}^{n} w_t(x_i)
\exp\prn*{-\eta_t\prn*{1-\alpha - C_{\varepsilon}(y_i,h_t(x_i))}} \right),
\end{align*}
where the last equality follows by telescoping the exponential reweighting
updates in \cref{alg:boosting-geo-med}.

For the exponential bound, fix a round $t$ and view
$Z_t(i)=1-\alpha-C_\varepsilon(y_i,h_t(x_i))$ as a random variable under
$w_t$. Since $\WL$ is a weak learner with parameter $\alpha+\zeta$,
$\En_{w_t} Z_t\geq \zeta$, and by definition its range length is at most $M$.
Since $\eta_t$ minimizes the one-dimensional objective in \eqref{alg:def:eta},
\Cref{lem:hoeffding-round-factor} gives
\[
\sum_i w_t(x_i)
\exp\!\Bigl(-\eta_t\bigl(1-\alpha-C_\varepsilon(y_i,h_t(x_i))\bigr)\Bigr)
\leq
\exp\!\left(-\frac{2\zeta^2}{M^2}\right).
\]
Multiplying over $t=1,\dots,T$ gives the stated exponential decay. Since
$L_{\divv,\beta\varepsilon}(f_T)$ is an empirical fraction, it must be
zero whenever the upper bound is smaller than $1/n$.
\end{proof}

\subsection{Generalization}
\label{sec:generalization}

The guarantee in \Cref{thm:error-bound} controls empirical exceedance loss on
the training sample. We now explain how the same geometric-median mechanism
leads to a population guarantee for finite base classes. The argument follows
the same route as the classical generalization analysis of median boosting
\citep{kegl2004medianconvergence} and margin-based boosting
\citep{bartlett1998boosting}: we
first measure whether the ensemble places enough aggregation weight near the
target, and then reduce the population bound to uniform convergence over a
finite class of sampled aggregates.

Let $(X,Y)\sim\cD$ be an independent test point. For a predictor $f$ and radius
$r\ge 0$, define the expected exceedance loss
\[
\bar L_{\divv,r}(f)
\;\ldef\;
\Pr_{(X,Y)\sim\cD}\{\divv(Y,f(X))>r\}.
\]
Recall that $L_{\divv,r}(f)$ denotes the corresponding empirical exceedance
loss on the training sample.

Throughout this section, whenever a geometric-median set is not a singleton, we
use one fixed measurable selector. This convention also applies to the
$\gamma$-perturbed geometric median, whose definition is itself a union of
geometric-median sets over perturbed ensembles. For concreteness, one may choose
an element of minimum norm, with a fixed deterministic
tie-breaking rule if needed. Thus each aggregation rule below is viewed as a
point-valued predictor.

We work with a known weak-learner class $\cH$, and each weak predictor produced
by \geomedboost\ belongs to this class. Let $h_1,\dots,h_T\in\cH$ be the weak
predictors produced by \geomedboost, and let
$\eta'_t=\eta_t/\sum_{s=1}^T\eta_s$ be the normalized aggregation weights. For
$\rho\in(0,1-\alpha)$, define the empirical concentration failure loss
\[
L^{(\rho)}_{\divv,\varepsilon}(f_T)
\;\ldef\;
\frac1n\sum_{i=1}^n
\indic\!\left\{
\sum_{t=1}^T
\eta'_t\indic\{\divv(y_i,h_t(x_i))\le \varepsilon\}
\le \alpha+\rho
\right\}.
\]
This quantity counts training points for which the ensemble fails to put an
$(\alpha+\rho)$ fraction of its aggregation weight inside the
$\varepsilon$-ball around the target. The stability implication itself only
needs the ensemble mass to be above the level at which the geometric median is
stable, while the extra $\rho$ is a buffer used by the finite-sampling argument
below. Thus the theorem uses $(\alpha-\gamma,\beta)$-stability and compares
population exceedance to this stronger empirical failure event with threshold
$\alpha+\rho$.

\begin{theorem}
\label{thm:finite-generalization}
Assume that $\cH$ is finite and that $\divv$ is
$(\alpha-\gamma,\beta)$-stable by geometric median. Fix
$\rho\in(0,1-\alpha)$. Then there exists a universal constant $c>0$
such that, with probability at least $1-\delta$ over the draw of the training
sample, every predictor $f_T$ output by \geomedboost satisfies
\[
\bar L_{\divv,\beta\varepsilon}(f_T)
\le
L^{(\rho)}_{\divv,\varepsilon}(f_T)
+
c\sqrt{\frac{\log n\log |\cH|}{\rho^2 n}}
+
c\sqrt{\frac{\log(1/\delta)}{n}}.
\]
\end{theorem}

This theorem controls population exceedance by the empirical concentration failure
loss.
The margin parameter $\rho$ is the price paid for replacing the
continuous aggregation weights $\eta'$ by finite empirical estimates in the
uniform convergence argument.

\begin{proof}[\pfref{thm:finite-generalization}]
Fix an output $f_T$ with weak predictors $h_1,\dots,h_T$ and aggregation
weights $\eta'=(\eta'_1,\dots,\eta'_T)$. For a point $(x,y)$, write
\[
U_{\eta'}(x,y)
\;\ldef\;
\sum_{t=1}^T
\eta'_t\indic\{\divv(y,h_t(x))\le \varepsilon\}.
\]
If $U_{\eta'}(x,y)>\alpha$, then every weighted ensemble in the
$\gamma$-neighborhood has more than an $\alpha-\gamma$ fraction of its mass in
$B_\varepsilon(y;\divv)$. By $(\alpha-\gamma,\beta)$-stability, every
admissible $\gamma$-perturbed geometric median lies in
$B_{\beta\varepsilon}(y;\divv)$. Hence
\[
\indic\{\divv(y,f_T(x))>\beta\varepsilon\}
\le
\indic\{U_{\eta'}(x,y)\le \alpha\}.
\]

Let $N$ be an integer to be chosen later. Sample
$J_1,\dots,J_N$ independently from $[T]$ according to $\eta'$, and define the
empirical estimate
\[
U_N(x,y)
\;\ldef\;
\frac1N\sum_{s=1}^N
\indic\{\divv(y,h_{J_s}(x))\le \varepsilon\}.
\]
For fixed $(x,y)$, this is an average of i.i.d.\ Bernoulli variables with mean
$U_{\eta'}(x,y)$. Hoeffding's inequality therefore gives
\[
\Pr_J\{U_N(x,y)>\alpha+\rho/2\mid U_{\eta'}(x,y)\le \alpha\}
\le
\exp(-N\rho^2/2),
\]
and
\[
\Pr_J\{U_N(x,y)\le \alpha+\rho/2\mid U_{\eta'}(x,y)>\alpha+\rho\}
\le
\exp(-N\rho^2/2).
\]
Averaging the first inequality over $(X,Y)\sim\cD$ yields
\[
\Pr_{\cD}\{U_{\eta'}(X,Y)\le \alpha\}
\le
\En_J\Pr_{\cD}\{U_N(X,Y)\le \alpha+\rho/2\}
+
\exp(-N\rho^2/2),
\]
while averaging the second over the training sample gives
\[
\En_J\frac1n\sum_{i=1}^n
\indic\{U_N(x_i,y_i)\le \alpha+\rho/2\}
\le
L^{(\rho)}_{\divv,\varepsilon}(f_T)
+
\exp(-N\rho^2/2).
\]

For an ordered tuple $(g_1,\dots,g_N)\in\cH^N$, define
\[
\phi_{g_1,\dots,g_N}(x,y)
\;\ldef\;
\indic\!\left\{
\frac1N\sum_{s=1}^N
\indic\{\divv(y,g_s(x))\le \varepsilon\}
\le \alpha+\rho/2
\right\}.
\]
There are at most $|\cH|^N$ such functions. A finite-class uniform convergence
bound implies that, with probability at least $1-\delta$, simultaneously for
all $(g_1,\dots,g_N)\in\cH^N$,
\begin{equation}
\label{ineq:uniform-conv-bound}
\Pr_{\cD}\{\phi_{g_1,\dots,g_N}(X,Y)=1\}
\le
\frac1n\sum_{i=1}^n \phi_{g_1,\dots,g_N}(x_i,y_i)
+
\sqrt{\frac{N\log|\cH|+\log(1/\delta)}{2n}}.
\end{equation}
Applying this bound to the random tuple $(h_{J_1},\dots,h_{J_N})$ and taking
expectation over $J$ gives
\[
\bar L_{\divv,\beta\varepsilon}(f_T)
\le
L^{(\rho)}_{\divv,\varepsilon}(f_T)
+2\exp(-N\rho^2/2)
+
\sqrt{\frac{N\log|\cH|+\log(1/\delta)}{2n}}.
\]
Choosing $N=\lceil 2\rho^{-2}\log n\rceil$ and absorbing constants into the
universal constant $c$ proves the claim.
\end{proof}

The same exponential-potential proof used for \Cref{thm:error-bound} also
controls the empirical concentration failure loss. If the boosting coefficients are chosen with the
shifted level $\alpha+\rho$,
then replacing $\alpha$ by $\alpha+\rho$ in the proof of
\Cref{thm:error-bound} gives exponential decay of
$L^{(\rho)}_{\divv,\varepsilon}(f_T)$. 
Thus the finite-class generalization
bound below has two parts: uniform convergence transfers the empirical
concentration failure loss to population exceedance
, and the empirical boosting can drive the empirical
concentration failure loss to zero in logarithmically many rounds.

\begin{corollary}
\label{cor:finite-generalization-margin}
Assume the conditions of \Cref{thm:finite-generalization}. Fix
$\zeta>0$ such that $\alpha+\rho+\zeta\le 1$. Suppose the boosting
coefficients are chosen from the shifted objective
\[
\eta_t
=\argmin_{\eta>0}
\sum_i w_t(x_i)
\exp\bigl(-\eta(1-\alpha-\rho-C_\varepsilon(y_i,h_t(x_i)))\bigr),
\]
and the weak learner has parameter $\alpha+\rho+\zeta$. Let
\[
M\;\ldef\;
\sup_{t,i} C_\varepsilon(y_i,h_t(x_i))
-
\inf_{t,i} C_\varepsilon(y_i,h_t(x_i)).
\]
Then
\[
L^{(\rho)}_{\divv,\varepsilon}(f_T)
\le
\exp\!\left(-\frac{2\zeta^2}{M^2}T\right).
\]
Consequently, if $T>M^2\log(n)/(2\zeta^2)$, then
$L^{(\rho)}_{\divv,\varepsilon}(f_T)=0$, and with probability at least
$1-\delta$,
\[
\bar L_{\divv,\beta\varepsilon}(f_T)
\le
c\sqrt{\frac{\log n\log |\cH|}{\rho^2 n}}
+
c\sqrt{\frac{\log(1/\delta)}{n}}.
\]
\end{corollary}

\begin{proof}[\pfref{cor:finite-generalization-margin}]
For a training point $(x_i,y_i)$, the event counted by
$L^{(\rho)}_{\divv,\varepsilon}(f_T)$ is
\[
\sum_{t=1}^T
\eta'_t\indic\{\divv(y_i,h_t(x_i))\le \varepsilon\}
\le \alpha+\rho.
\]
On this event, at least a $(1-\alpha-\rho)$ fraction of the aggregation weight
lies outside $B_\varepsilon(y_i;\divv)$. Since
$C_\varepsilon(y,y')\ge \indic\{\divv(y,y')>\varepsilon\}$, we have
\[
\sum_{t=1}^T
\eta_t(1-\alpha-\rho-C_\varepsilon(y_i,h_t(x_i)))\le 0.
\]
The rest is the exponential-potential proof of \Cref{thm:error-bound} with
$\alpha$ replaced by $\alpha+\rho$:
\[
L^{(\rho)}_{\divv,\varepsilon}(f_T)
\le
\prod_{t=1}^{T}
\sum_i w_t(x_i)
\exp\bigl(-\eta_t(1-\alpha-\rho-C_\varepsilon(y_i,h_t(x_i)))\bigr).
\]
Under the weak learner parameter $\alpha+\rho+\zeta$, the random variable
$1-\alpha-\rho-C_\varepsilon(y_i,h_t(x_i))$ has weighted mean at least
$\zeta$ and range length at most $M$. Applying
\Cref{lem:hoeffding-round-factor} to each round gives the exponential bound.
Once this bound is below $1/n$, the empirical fraction
$L^{(\rho)}_{\divv,\varepsilon}(f_T)$ must be zero. The final display follows
from \Cref{thm:finite-generalization}.
\end{proof}

Together with a boundedness assumption, we have the following bound on the expected divergence.

\begin{corollary}
\label{cor:finite-generalization-risk}
Assume the conditions of \Cref{cor:finite-generalization-margin}. Suppose, in
addition, that the divergence is uniformly bounded: there is a constant
$D_{\max}<\infty$ such that
\[
\divv(y,y')\le D_{\max}
\qquad \forall\,y,y'\in\cY.
\]
Then, with probability at least $1-\delta$,
\[
\En_{(X,Y)\sim\cD}\brk*{\divv(Y,f_T(X))}
\le
\beta\varepsilon
+
D_{\max}\left(
c\sqrt{\frac{\log n\log |\cH|}{\rho^2 n}}
+
c\sqrt{\frac{\log(1/\delta)}{n}}
\right).
\]
\end{corollary}

For infinite weak-learner classes, the same proof strategy remains available
once this counting step is replaced by a covering-number step.

\section{Stability Results for Specific Divergences}
\label{sec:stability-results}

We now verify how the stability condition behaves for the divergences studied in this paper.
Identifying the corresponding $(\alpha,\beta)$-stability thresholds and properties is a technical contribution of this work.
We begin with vector-valued prediction under $\ell_1$ and $\ell_2$, and then turn to conditional
density estimation under $\TV$, $\Hel$, and $\KL$.

\subsection{Stability for Vector-Valued Prediction}
\label{sec:boostability-for-vec-prediction}

\begin{proposition}[$\ell_1$ stability]
\label{prop:l1-boostability}
Let $\ell_1(y,z)=\|y-z\|_1\ldef \sum_i |y_i-z_i|$ on $\R^d$.
\begin{enumerate}[label=(\roman*),leftmargin=1.6em]
\item For any $\alpha>1/2$, $\ell_1$ is $(\alpha,d)$-stable by geometric median:
if $\eta(B_\eps(z;\ell_1))>1/2$, then $\med_1(\labelspace,\eta)\subseteq B_{d\eps}(z;\ell_1)$.
\item For any $\alpha>1/2$ and any $\beta<d$, $\ell_1$ is not $(\alpha,\beta)$-stable by geometric median.
\end{enumerate}
\end{proposition}

\begin{proof}[\pfref{prop:l1-boostability}]
We use the following standard characterization of weighted medians on the line.

\begin{lemma}[1-dimensional weighted median]
\label{lem:1d-median}
If an interval $I\subset \bR$ has $\eta(I)>1/2$, then every weighted median lies in $I$.
\end{lemma}

\begin{proof}[\pfref{lem:1d-median}]
In one dimension, $m$ minimizes $\sum_t\eta_t|u-y_t|$ iff the weight strictly to the left of $m$ and the weight strictly to the right of $m$ are both at most $1/2$. If $m$ lies outside an interval carrying weight $>1/2$, one of these two inequalities is violated.
\end{proof}

Let $y_t = (y_{t1},...,y_{td})$ for all $t\in [T]$.
The $\ell_1$ objective separates coordinate-wise:
\[
\sum_{t=1}^T\eta_t\|g-y_t\|_1
=\sum_{j=1}^d\sum_{t=1}^T\eta_t|g_j-y_{tj}|,
\]
so each coordinate of an $\ell_1$ geometric median is a one-dimensional weighted median of the corresponding coordinate sample. If $\eta(B_\eps(z;\ell_1))>1/2$, then, for every $j$,
\[
\eta(\set{y_t:y_{tj}\in[z_j-\eps,z_j+\eps]})
\ge \eta(B_\eps(z;\ell_1))>1/2.
\]
By \cref{lem:1d-median}, every median coordinate lies in $[z_j-\eps,z_j+\eps]$. Hence, for every $g^\star\in\med_1(\labelspace,\eta)$,
\[
\|g^\star-z\|_1=\sum_{j=1}^d |g^\star_j-z_j|\le d\varepsilon,
\]
which proves claim $(i)$.

For claim $(ii)$, fix $z\in\R^d$, $\varepsilon>0$,
$\delta\in(0,\tfrac12)$, and $L>\eps$. Let
$e_t$ be the $t$-th standard basis vector in $\R^d$, and let
$\indic=(1,\dots,1)\in\R^d$. Take $y_t=z+\varepsilon e_t$ with weight
$\eta_t=(1+2\delta)/(2d)$ for $t\le d$, and
$y_{d+1}=z+L\indic$ with weight $\eta_{d+1}=(1-2\delta)/2$.
Then
\[
\eta(B_\eps(z;\ell_1))=\sum_{t=1}^d \eta_t=\tfrac12+\delta>\tfrac12.
\]
For each coordinate $j$, the mass below $z_j+\varepsilon$ is $(d-1)(1+2\delta)/(2d)$, the mass at $z_j+\varepsilon$ is $(1+2\delta)/(2d)$, and the mass above it is $1/2-\delta$. Thus $z_j+\varepsilon$ is the unique weighted median in coordinate $j$, and hence
\[
\med_1(\labelspace,\eta)=\{z+\varepsilon\indic\}.
\]
The resulting median has distance $d\varepsilon$ from $z$, so no enlargement factor $\beta<d$ can guarantee inclusion in $B_{\beta\eps}(z;\ell_1)$.

\end{proof}

This $\ell_1$ guarantee is available as soon as the weighted predictions have
majority concentration after the $\gamma$ perturbation. 
The cost of the flexible $\ell_1$ stability
condition is the dimension-dependent enlargement factor $d$. The Euclidean case
has a different shape: it replaces this dimension dependence with a
dimension-free stability tradeoff between the required concentration level and
the enlargement factor.

\begin{proposition}[$\ell_2$ stability]
\label{prop:l2-boostability}
Let $d\ge 2$ and $\ell_2(y,z)=\|y-z\|_2\ldef \sqrt{\sum_i (y_i-z_i)^2}$ on $\R^d$. For any $\beta>1$, define
\[
\alpha_{2}(\beta)\;\ldef\;\frac{\beta}{\beta+\sqrt{\beta^2-1}}.
\]
\begin{enumerate}[label=(\roman*),leftmargin=1.6em]
\item If $\alpha>\alpha_{2}(\beta)$, then $\ell_2$ is $(\alpha,\beta)$-stable by geometric median.
\item If $\alpha<\alpha_{2}(\beta)$, then $\ell_2$ is not $(\alpha,\beta)$-stable by geometric median.
\end{enumerate}
\end{proposition}

\begin{proof}[\pfref{prop:l2-boostability}]
We first record two elementary facts used in the proof.

\begin{lemma}
\label{lem:direction-component}
For any $z,x,g\in \R^d$, $\eps>0$, and $\beta>1$, if $x\in B_{\eps}(z;\ell_2)$ and $g\notin B_{\beta\eps}(z;\ell_2)$, then
\[
\inner{g-x,\,g-z}
\;\ge\;
\frac{\sqrt{\beta^2-1}}{\beta}\,\norm{g-x}_2\,\norm{g-z}_2.
\]
\end{lemma}

\begin{proof}[\pfref{lem:direction-component}]
Let $a=g-z$ and $b=x-z$. If $b=0$, the normalized inner product below is
$1$, so assume $b\neq 0$. Set
$r=\|b\|_2/\|a\|_2<1/\beta$ and
$t=\langle a,b\rangle/(\|a\|_2\|b\|_2)\in[-1,1]$. Since $g-x=a-b$,
\[
\frac{\langle g-x,\, g-z\rangle}{\|g-x\|_2\,\|g-z\|_2}
= \frac{1 - r t}{\sqrt{1 + r^2 - 2 r t}}.
\]
For fixed $r\in(0,1)$, the right-hand side is minimized over $t\in[-1,1]$
at $t=r$, where it equals $\sqrt{1-r^2}$. Therefore
\[
\frac{\langle g-x,\, g-z\rangle}{\|g-x\|_2\,\|g-z\|_2}
\;\ge\; \sqrt{1 - r^2}
\;>\; \sqrt{1 - \frac{1}{\beta^2}}
= \frac{\sqrt{\beta^2 - 1}}{\beta}.
\]
\end{proof}

\begin{lemma}
\label{lem:example-l2}
For any $\eps>0$, $\alpha\in (1/2,1)$, let $\theta$ satisfy $\alpha\cos\theta = 1-\alpha$. Let
\[
p_\pm=(\varepsilon\sin\theta,\ \pm \varepsilon\cos\theta),
\qquad
q=\Big(\frac{2\varepsilon}{\sin\theta},\,0\Big),
\]
with weights $\eta(p_+)=\eta(p_-)=\alpha/2$ and $\eta(q)=1-\alpha$.  Then the weighted geometric median is
\[
g^\star = \Big(\frac{\varepsilon}{\sin\theta},\,0\Big).
\]
\end{lemma}

\begin{proof}[\pfref{lem:example-l2}]
Let $F(g)=\tfrac{\alpha}{2}\|g-p_+\|+\tfrac{\alpha}{2}\|g-p_-\|+(1-\alpha)\|g-q\|$.
Note that $p_+$ and $p_-$ are reflections of each other across the $x$-axis, while $q$ lies on the $x$-axis.
For any point $g=(x,y)$, let $\tilde g = (x,-y)$ and $\bar g = \tfrac12(g+\tilde g) = (x,0)$. Then, since $F$ is symmetric with respect to the x-axis and by the convexity of the norm and of $F$,
\[
F(\bar g) \le \tfrac12F(g)+\tfrac12F(\tilde g) = F(g).
\]
Thus, for every $g$, there exists a point $\bar g$ on the $x$-axis with $F(\bar g)\le F(g)$. Furthermore, if $F(\bar g) = F(g)$ for some $y\neq 0$, then there on this line, $0$ is always a subgradient, which contradicts the structure of the norm function. Hence, every minimizer of $F$ must lie on the $x$-axis.
Write such a minimizer as $g=(x,0)$.

\medskip

For $g\neq p_\pm,q$, the objective is differentiable and the first-order condition is
\[
0 = \frac{\alpha}{2}\frac{g-p_+}{\|g-p_+\|}
+ \frac{\alpha}{2}\frac{g-p_-}{\|g-p_-\|}
+ (1-\alpha)\frac{g-q}{\|g-q\|}.
\]
The $y$-components cancel, yielding
\[
\alpha \, u_x = 1-\alpha,
\]
where $u_x$ is the $x$-component of $(g-p_\pm)/\|g-p_\pm\|$.
Evaluating at $g_0=(\varepsilon/\sin\theta,0)$ gives $u_x=\cos\theta$, so the above condition holds exactly when $\alpha\cos\theta=1-\alpha$, which is our assumption.
Thus $g_0$ is a stationary point.
Since $F$ is strictly convex along the $x$-axis away from $p_\pm,q$, $g_0$ is the unique minimizer.
\end{proof}

Write $F(g)=\sum_{t=1}^T \eta_t\|g-y_t\|_2$. We show first that if $\eta(B_\eps(z;\ell_2))>\alpha$, then no point outside $B_{\beta\eps}(z;\ell_2)$ can minimize $F$.
For any $g\notin \labelspace$, $F$ is differentiable at $g$ with
\[
\nabla F(g)=\sum_{t=1}^T \eta_t\,\frac{g-y_t}{\|g-y_t\|_2}.
\]
Specifically, we show that for any $g\notin ( B_{\beta\eps}(z;\ell_2)\bigcup \labelspace)$, the inner product $\inner{\nabla F(g), \frac{g-z}{\norm{g-z}_2}} \neq 0$, which implies $\nabla F(g)\neq 0$ and $g\notin \med_2(\labelspace, \eta)$.
For any $y_t\in B_\eps(z;\ell_2)$, we have by \cref{lem:direction-component},
\begin{align}
\label{ineq:in-the-ball-product}
    \inner*{\frac{g-y_t}{\|g-y_t\|_2},\frac{g-z}{\norm{g-z}_2}} \geq  \frac{\sqrt{\beta^2-1}}{\beta}.
\end{align}
Moreover, for any $y_t\notin B_\eps(z;\ell_2)$, we have
\begin{align}
\label{ineq:out-of-the-ball-product}
    \inner*{\frac{g-y_t}{\|g-y_t\|_2},\frac{g-z}{\norm{g-z}_2}} \geq - 1.
\end{align}
Overall, combining \eqref{ineq:in-the-ball-product} and \eqref{ineq:out-of-the-ball-product}, for any $g\notin ( B_{\beta\eps}(z;\ell_2)\bigcup \labelspace)$ we have
\begin{align*}
    \inner*{\nabla F(g), \frac{g-z}{\norm{g-z}_2}}&=\sum_{t=1}^T \eta_t \inner*{\frac{g-y_t}{\|g-y_t\|_2},\frac{g-z}{\norm{g-z}_2}}\\
    &=\sum_{t:y_t\in B_\eps(z;\ell_2)} \eta_t \inner*{\frac{g-y_t}{\|g-y_t\|_2},\frac{g-z}{\norm{g-z}_2}} + \sum_{t:y_t\notin B_\eps(z;\ell_2)} \eta_t \inner*{\frac{g-y_t}{\|g-y_t\|_2},\frac{g-z}{\norm{g-z}_2}}\\
    &\geq \eta(B_\eps(z;\ell_2)) \cdot \frac{\sqrt{\beta^2-1}}{\beta} - (1 - \eta(B_\eps(z;\ell_2)))\\
    &= \eta(B_\eps(z;\ell_2)) \cdot \frac{\beta + \sqrt{\beta^2-1}}{\beta} - 1 >0,
\end{align*}
where the last inequality follows from $\eta(B_\eps(z;\ell_2))>\alpha>\beta/(\beta + \sqrt{\beta^2-1})$. Similar arguments can be made to show that for $g\in \labelspace\setminus B_{\beta\eps}(z;\ell_2)$, we have $0\notin \partial F(g)$.
These conclude the proof of $(i)$ by showing that $\med_2(\labelspace, \eta) \subseteq \set{g: \nabla F(g) = 0 ~\text{or}~0\in \partial F(g) } \subseteq B_{\beta\eps}(z;\ell_2)$.

For the proof of $(ii)$, we consider the example in \Cref{lem:example-l2}, which shows that the guarantee in Proposition~\ref{prop:l2-boostability} is sharp.
Indeed, observe that $\|p_\pm\|_2 = \eps$, so $p_\pm\in B_\eps(0)$.
Moreover,
\[
\eta(B_\eps(0)) = \eta(p_+)+\eta(p_-) = \alpha.
\]
However, since $\sin\theta < 1$, we have $\|g^\star\|_2 = \varepsilon/\sin\theta > \varepsilon$.
More generally, for any $\beta>1$, the condition
\[
\frac{\varepsilon}{\sin\theta} > \beta\varepsilon
\qquad\Longleftrightarrow\qquad
\sin\theta < \frac{1}{\beta}
\]
ensures that $g^\star\notin B_{\beta\varepsilon}(0)$.
Using $\alpha\cos\theta=1-\alpha$ and $\sin^2\theta+\cos^2\theta=1$, one checks that
\[
\sin\theta<\frac{1}{\beta}
\quad\Longleftrightarrow\quad
\alpha < \frac{\beta}{\beta + \sqrt{\beta^2-1}}.
\]
This concludes the proof of claim $(ii)$.
\end{proof}

Together, \cref{prop:l1-boostability,prop:l2-boostability} show how the choice
between $\ell_1$ and $\ell_2$ changes the resulting boosting guarantee. For
$\ell_1$, a majority concentration condition is enough: the stability threshold
is simply $\alpha-\gamma>1/2$. The cost is the dimension factor in the final
exceedance radius.
For $\ell_2$, the exceedance radius is instead
$\beta\varepsilon$, with no dependence on $d$, but this dimension-free radius
requires the stronger concentration condition
$\alpha-\gamma>\alpha_2(\beta)$. As $\beta$ approaches $1$, the enlarged radius
approaches the original weak radius $\varepsilon$, and the required
concentration level correspondingly increases. 

\subsection{Stability for Conditional Density Estimation}
\label{sec:boostability-for-density-estimation}

For conditional density estimation, 
our goal is more limited: 
we give sufficient
($\alpha$, $\beta$)-stability conditions rather than sharp characterizations of
the geometric-median threshold. We establish such positive results for total
variation and Hellinger distance. Square-root KL behaves differently: it is not
directly stable under its own geometric median, but KL guarantees can still be
obtained indirectly by aggregating in Hellinger distance and then converting the
resulting bound.

\begin{proposition}[Total variation]
\label{prop:TV-boostability}
Let $d\geq 2$. For any $\alpha>1/2$, the total variation distance $\TV(p,q)\ldef \frac{1}{2}\norm{p-q}_1$ on $\Delta_d$ is $(\alpha,2(d-1))$-stable by geometric median.
\end{proposition}

\paragraph{Equal proportion across coordinates.}
The following lemma formalizes the observation that, at any weighted TV (or $\ell_1$) geometric median under the simplex constraint, all active coordinates share a common imbalance proportion.

\begin{lemma}[Equal proportion of mass across coordinates]
\label{lem:equal-proportion}
Let $F(g)=\tfrac12\sum_{t=1}^T \eta_t\|y_t-g\|_1$ with $\sum_t \eta_t=1$, and suppose $g\in\Delta_d$ minimizes $F$ over the simplex $\Delta_d=\{g\in\R_+^d:\sum_j g_j=1\}$.
For each coordinate $j$, define the one–sided weights
\[
W_j^+(t)\coloneqq \sum_{t':\,y_{t'j}\leq t} \eta_{t'},
\qquad
W_j^-(t)\coloneqq \sum_{t':\,y_{t'j}<t} \eta_{t'}.
\]
Then 
 \[\bigcap_{j=1}^d [W_j^-(g_j),\,W_j^+(g_j)]\neq\emptyset.\]
\end{lemma}

\begin{proof}[\pfref{lem:equal-proportion}]
Write $F(g)=\sum_{j=1}^d f_j(g_j)$ with
\(
f_j(u)=\tfrac12\sum_{t=1}^T \eta_t|u-y_{tj}|.
\)
Each $f_j$ is convex and piecewise linear.  A standard 1-D calculation gives the one-sided directional derivatives at $u$:
\[
f'_+(u)=\tfrac12\big(1-2W_j^+(u)\big),\qquad
f'_-(u)=\tfrac12\big(1-2W_j^-(u)\big).
\]
Equivalently, the subdifferential interval is
\begin{align}
\label{eq:partial-diff}
 \partial f_j(u)=[f'_+(u), f'_-(u)] =\tfrac12\,[\,1-2W_j^+(u),\;1-2W_j^-(u)\,].
\end{align}

Fix any two coordinates $j\neq k$ and consider the feasible exchange line
\(
h_{jk}(t)\coloneqq F(g+t e_j - t e_k)
\)
for small $t$ so that $g+t e_j-t e_k\in\Delta_d$.
Since $g$ minimizes $F$ over $\Delta_d$, it minimizes the convex function $h_{jk}$ at $t=0$.
Thus the first-order optimality for convex functions yields
\(
0\in \partial h_{jk}(0).
\)
By separability and the chain rule for subgradients along affine maps,
\[
\partial h_{jk}(0)=\partial f_j(g_j)-\partial f_k(g_k).
\]
Hence there exist $s_j\in\partial f_j(g_j)$ and $s_k\in\partial f_k(g_k)$ with
\(
s_j-s_k=0
\)
(i.e., \(s_j=s_k\)).  Using \eqref{eq:partial-diff}, any $s_j\in\partial f_j(g_j)$ can be written as
\[
s_j \;=\; \tfrac12\big(1-2\theta_j\big)
\quad\text{for some }\;
\theta_j\in[W_j^-(g_j), W_j^+(g_j)],
\]
and similarly \(s_k=\tfrac12(1-2\theta_k)\) with \(\theta_k\in[W_k^-(g_k),W_k^+(g_k)]\).
Since \(s_j=s_k\), we must have \(\theta_j=\theta_k\).  Denote this common value by \(\tau\).
As the choice of the pair \((j,k)\) was arbitrary, the same argument applied to any pair shows that the same \(\tau\) belongs to every interval \([W_\ell^-(g_\ell),\,W_\ell^+(g_\ell)]\), $\ell=1,\dots,d$.
Therefore
\[
\tau\in\bigcap_{\ell=1}^d [W_\ell^-(g_\ell),\,W_\ell^+(g_\ell)],
\]
proving the intersection is nonempty.
\end{proof}

Intuitively, this lemma says that the weighted geometric median $g$ balances all coordinates at a common proportion of cumulative weight—each coordinate's one-dimensional distribution is “cut” at a point $g_j$ where the fraction of total mass on either side is the same across all $j$.

\begin{proof}[\pfref{prop:TV-boostability}]
Let $F(g)\coloneqq \sum_{t=1}^T \eta_t\,\TV(y_t,g)$ and suppose $g\in\med_\TV(\labelspace,\eta)$.
Assume $\eta(B_\eps(z;\TV))>\tfrac12$ and invoke \Cref{lem:equal-proportion} to obtain
\[
\tau \in \bigcap_{j=1}^d [\,W_j^-(g_j),\,W_j^+(g_j)\,].
\]
Without loss of generality (by symmetry of the “left/right” roles), assume $\tau\ge\tfrac12$.
Then for every coordinate $j$ we have
\[
W_j^+(g_j)\ \ge\ \tau\ \ge\ \tfrac12.
\tag{1}\label{eq:right-mass-half}
\]

\begin{lemma}
\label{lem:TV-coord-within-proof}
If $g_j<z_j$, then $|g_j-z_j|\leq 2\varepsilon$.
\end{lemma}

\begin{proof}[\pfref{lem:TV-coord-within-proof}]
Suppose, for contradiction, that for some coordinate $j$ we have $g_j<z_j$ and $z_j-g_j>2\varepsilon$.
Consider the $\TV$–ball $B_\varepsilon(z;\TV)=\{x:\tfrac12\|x-z\|_1\le\varepsilon\}.$
For any $y_t\in B_\varepsilon(z;\TV)$, we must have $|y_{tj}-z_j|\le 2\varepsilon$, hence $y_{tj}\ge z_j-2\varepsilon>g_j$.
That is, every such $y_t$ satisfies $y_{tj}>g_j$, so all points within $B_\varepsilon(z;\TV)$ (whose total weight exceeds $\tfrac12$) project into the interval $(g_j,1]$ along coordinate $j$.
Therefore,
\[
W_j^+(g_j)
=\sum_{t:y_{tj}\le g_j} \eta_t
<1-\eta(B_\varepsilon(z;\TV))
<\tfrac12,
\]
contradicting \eqref{eq:right-mass-half}.
Hence, $z_j-g_j<2\varepsilon$ whenever $g_j<z_j$.
\end{proof}

Because $g,z\in\Delta_d$, we have $\sum_{j=1}^d(g_j-z_j)=0$, so
\[
\TV(g,z)=\tfrac12\sum_{j=1}^d|g_j-z_j|
=\sum_{j:\,g_j<z_j}(z_j-g_j).
\]
By Lemma~\ref{lem:TV-coord-within-proof}, each summand on the right is $\leq 2\varepsilon$ and there are in total at most $d-1$ summands, hence
\[
\TV(g,z)
=\sum_{j:\,g_j<z_j}(z_j-g_j)
\leq 2(d-1)\varepsilon.
\]
Therefore $\med_\TV(\labelspace,\eta)\subseteq B_{2(d-1)\varepsilon}(z;\TV)$, i.e., TV is $(\alpha,2(d-1))$-stable by geometric median for any $\alpha>\tfrac12$.
\end{proof}

As with $\ell_1$, total variation only requires majority concentration, but
the final radius grows with the simplex dimension. Similar to $\ell_2$ through the square-root embedding, Hellinger distance gives a dimension-free tradeoff. For an enlarged radius of $\beta\varepsilon$, the concentration threshold increases compared to $\ell_2$.\looseness=-1

\begin{proposition}[Hellinger]
\label{prop:hel-boostability}
Let $d\geq 2$. For $\beta>1$ define
\[
\alpha_{\Hel}(\beta,d)\;\ldef\;
\begin{cases}
\frac{2}{3}, & d=2,\\[0.5em]
\frac{2\beta^2}{3\beta^2-1}, & d\ge 3.
\end{cases}
\]
If $\alpha>\alpha_{\Hel}(\beta,d)$, then the Hellinger distance $\Hel(p,q)\ldef \sqrt{\frac{1}{2} \sum_{j=1}^d (\sqrt{p_j}-\sqrt{q_j})^2 } $ on $\Delta_d$ is $(\alpha,\beta)$-stable by geometric median.
\end{proposition}

\begin{lemma}
\label{lem:hellinger-direction-component}
Let $g,y,z \in \Delta^d$ be probability distributions and let $\sqrt{\cdot}$ denote element-wise square root.
For any $\beta>1$,  suppose $\Hel(z,g)>\beta \Hel(z,y)$. Let
\begin{align*}
    p_y = \sqrt{g}-\sqrt{y} -  \inner{\sqrt{g}-\sqrt{y},\sqrt{g}} \sqrt{g} \text{~~and~~}
    p_z = \sqrt{g}-\sqrt{z} -  \inner{\sqrt{g}-\sqrt{z},\sqrt{g}} \sqrt{g}
\end{align*}
be the tangent projection of $\sqrt{g}-\sqrt{y}$ and $\sqrt{g}-\sqrt{z}$ at point $\sqrt{g}$. Then we have
\[
\inner*{p_y ,p_z}
\;\ge\;
\begin{cases}
\displaystyle \Hel(y,g)\,\Hel(z,g), & d=2,\\[1.2em]
\displaystyle \frac{\beta^2-1}{\beta^2}\,\Hel(y,g)\,\Hel(z,g), & d\ge 3.
\end{cases}
\]
\end{lemma}

\begin{proof}[\pfref{lem:hellinger-direction-component}]
let
\[\langle \sqrt{g}, \sqrt{y}\rangle=\cos a,\quad\langle \sqrt{g}, \sqrt{z}\rangle=\cos b,\quad\text{and}\quad\langle \sqrt{y}, \sqrt{z}\rangle=\cos c,\]
where $0\leq a, b, c\leq \frac{\pi}{2}$.
Then we have
\begin{align*}
    \Hel(y,g)=\sqrt{1-\cos a},\quad \Hel(z, g)=\sqrt{1-\cos b}\quad\text{and}\quad\Hel(z, y)=\sqrt{1-\cos c}.
\end{align*}
Thus, we have
\begin{align*}
    \inner{p_y,p_z}  &= \inner{\sqrt{g}-\sqrt{y},\sqrt{g}-\sqrt{z}} - \inner{\sqrt{g}-\sqrt{y},\sqrt{g}}\inner{\sqrt{g}-\sqrt{z},\sqrt{g}}\\
    &= 1-\cos a-\cos b + \cos c - (1-\cos a)(1-\cos b)\\
    &= \cos c - \cos a\cos b.
\end{align*}
Thus, our ratio of interest is
\begin{align*}
    \frac{\langle p_y, p_z\rangle}{\Hel(y,g)\Hel(z,g)}=\frac{\cos c - \cos a\cos b}{\sqrt{(1-\cos a)(1-\cos b)}}
\end{align*}

When $d=2$,
\begin{align*}
    \inner{\sqrt{y},\sqrt{z}} = \cos (a-b) \text{~~or~~} \cos (a+b).
\end{align*}
But since we have
\begin{align*}
     \frac{\Hel(z,g)}{\Hel(z,y)} = \sqrt{\frac{1-\cos b}{1- \inner{\sqrt{y},\sqrt{z}} }} >  \beta > 1,
\end{align*}
we further have $\inner{\sqrt{y},\sqrt{z}}>\cos b$. Thus $ \inner{\sqrt{y},\sqrt{z}} = \cos (a-b)$.
Then when $d=2$, we have
\begin{align*}
\frac{\langle p_y, p_z\rangle}{\Hel(y,g)\Hel(z,g)}&=\frac{\cos(a-b) - \cos a\cos b}{\sqrt{(1-\cos a)(1-\cos b)}} \\
&=\frac{\sin a \cdot \sin b }{\sqrt{1-\cos a }\sqrt{1-\cos b }}.
\end{align*}
Since $\frac{\sin x }{\sqrt{1-\cos x }} = \sqrt{2}\cos(\frac{x}{2})\geq \sqrt{2}\cos(\frac{\pi}{4}) = 1$ for $x\in [0, \frac{\pi}{2}]$.
Thus for $d=2$, we have
\begin{align*}
    \inner*{p_y ,p_z} \geq  \Hel(y,g)\Hel(z,g).
\end{align*}

When $d\geq3$,
let
\begin{align*}
    u = \frac{\Hel(z,g)}{\Hel(z,y)} = \sqrt{\frac{1-\cos b}{1-\cos c}} >  \beta > 1.
\end{align*}
This implies $\cos c> \cos b$ and $\cos b < 1$.
Let $h(x) = \frac{x^2-1}{x^2}$. Since $h(x)$ is monotonically increasing, we have $h(u)>h(\beta)$. We are going to show our target inequality as the first inequality below that
\begin{align}
\label{ineq:target}
    \frac{\langle p_y, p_z\rangle}{\Hel(y,g)\Hel(z,g)}=\frac{\cos c - \cos a\cos b}{\sqrt{(1-\cos a)(1-\cos b)}} \geq h(u) = \frac{\cos c-\cos b}{(1-\cos b)} > \frac{\beta^2-1}{\beta^2}.
\end{align}
This will imply our desired result for $d\geq 3$. Let
\begin{align*}
    f(x) = \frac{\cos c - x\cos b}{\sqrt{1-x}}.
\end{align*}
The derivative is
\begin{align*}
    f'(x) &= - \frac{\cos b}{\sqrt{1-x}}  + \frac{\cos c - x\cos b}{2(1-x)^{3/2}} \\
    &= \frac{x\cos b + \cos c - 2\cos b}{2(1-x)^{3/2}} .
\end{align*}

\noindent\textbf{Case I}: If $\cos c \geq 2 \cos b$, then $f(x)$ is monotonically increasing on $x\geq 0$. This implies
\begin{align*}
   \frac{\cos c - \cos a\cos b}{\sqrt{1-\cos a}} =  f(\cos a) \geq f(0) = \cos c .
\end{align*}
In this case, our target \eqref{ineq:target} 
can be proven by the last inequality below
\begin{align*}
    \frac{\langle p_y, p_z\rangle}{\Hel(y,g)\Hel(z,g)}&=\frac{\cos c - \cos a\cos b}{\sqrt{(1-\cos a)(1-\cos b)}} \\
    &= \frac{f(\cos a)}{\sqrt{1-\cos b}} \\
    &\geq \frac{f(0)}{\sqrt{1-\cos b}}\\
    &=\frac{\cos c}{\sqrt{1-\cos b}} \\
    &\geq \frac{\cos c-\cos b}{1-\cos b}.
\end{align*}
Here the inequality $f(\cos a)\ge f(0)$ uses the monotonicity of $f$ in Case I,
and the final inequality is 
equivalent to
\begin{align*}
    \cos^2 c(1-\cos b) \geq (\cos c - \cos b)^2.
\end{align*}
This is again equivalent to
\begin{align*}
    2\cos b \cos c \geq \cos^2 c \cos b +\cos^2 b.
\end{align*}
This inequality is true because $1\geq \cos c> \cos b$. This concludes the proof in Case I.

\bigskip

\noindent\textbf{Case II}: If $\cos c < 2\cos b$, then $f(x)$ obtains its minimum at the point $\frac{2\cos b-\cos c}{\cos b}$. Thus our target \eqref{ineq:target} can be proven by the last inequality below
\begin{align*}
    \frac{\langle p_y, p_z\rangle}{\Hel(y,g)\Hel(z,g)}&=\frac{\cos c - \cos a\cos b}{\sqrt{(1-\cos a)(1-\cos b)}} \\
    &= \frac{f(\cos a)}{\sqrt{1-\cos b}} \\
    &\geq f\prn*{\frac{2\cos b-\cos c}{\cos b}}\cdot \frac{1}{\sqrt{1-\cos b}}\\
    &=2\sqrt{(\cos c-\cos b)\cos b}  \\
    &\geq \frac{\cos c-\cos b}{(1-\cos b)}.
\end{align*}
Here the first inequality uses the minimum of $f$ in Case II, and the final
inequality is equivalent to
\begin{align*}
    5\cos b - 4\cos^2 b \geq \cos c .
\end{align*}
This inequality holds because
\begin{align*}
    5\cos b - 4\cos^2 b \geq  2\cos b\cdot \indic( \cos b \leq 3/4) + \indic(\cos b > 3/4) \geq \cos c.
\end{align*}
This concludes the proof in Case II. Thus concludes our proof.
\end{proof}

\begin{proof}[\pfref{prop:hel-boostability}]
Let $\sphere\subset \bR^d$ be the $d-1$ dimensional sphere. Let $F(\sqrt{g})=\sum_{t=1}^T \eta_t \Hel(y_t,g) = \frac{\sqrt{2}}{2} \sum_{t=1}^T \eta_t \norm{\sqrt{g}-\sqrt{y_t}}_2$ be the target function on the sphere \sphere. We are going to show that if $\eta(B_\eps(z;\Hel))>\alpha$, then for any $g\notin ( B_{\beta\eps}(z;\Hel)\bigcup \labelspace)$, we have $\nabla F(\sqrt{g}) \not\parallel \sqrt{g}$ and for $g\in \labelspace\setminus B_{\beta\eps}(z;\Hel)$, we have $t\sqrt{g}\notin \partial F(\sqrt{g})$ for any $t\in \bR$.
For any $\sqrt{g}\notin \labelspace$, $F$ is differentiable at $\sqrt{g}$ with
\[
\nabla F(\sqrt{g})= \frac{\sqrt{2}}{2} \sum_{t=1}^T \eta_t\,\frac{\sqrt{g}-\sqrt{y_t}}{\norm{\sqrt{g}-\sqrt{y_t}}_2}.
\]
Then we can consider the projection onto the tangent space at $\sqrt{g}$, which is
\begin{align*}
    (\nabla F(\sqrt{g}))_{\parallel}
    &=\nabla F(\sqrt{g}) - \inner{\nabla F(\sqrt{g}), \sqrt{g}} \sqrt{g} \\
    &=\sum_{t=1}^T \frac{\sqrt{2}\eta_t}{2\norm{\sqrt{g}-\sqrt{y_t}}_2}
    \prn*{\sqrt{g}-\sqrt{y_t} - \inner{\sqrt{g}-\sqrt{y_t},\sqrt{g}}\sqrt{g}}\\
    &= \sum_{t=1}^T \frac{\sqrt{2}\eta_t}{2\norm{\sqrt{g}-\sqrt{y_t}}_2}  p_{y_t},
\end{align*}
where $p_{y_t} = \sqrt{g}-\sqrt{y_t} - \inner{\sqrt{g}-\sqrt{y_t},\sqrt{g}}\sqrt{g}$. Let $p_z = \sqrt{g}-\sqrt{z} -  \inner{\sqrt{g}-\sqrt{z},\sqrt{g}} \sqrt{g}$.
Specifically, we show that for any $g\notin ( B_{\beta\eps}(z;\Hel)\bigcup \labelspace)$, the inner product $\inner*{ (\nabla F(\sqrt{g}))_{\parallel} , \frac{p_z}{\norm{\sqrt{g}-\sqrt{z}}_2}} \neq 0$, which implies $\sqrt{g}\not\parallel \nabla F(\sqrt{g})$ and $g\notin \med_\Hel(\labelspace, \eta)$.
For any $y_t$, we have by \cref{lem:hellinger-direction-component},
\begin{align}
\label{ineq:hellinger-in-the-ball-product}
    \inner*{p_{y_t},p_z} \geq
    \begin{cases}
    \displaystyle
    \frac{1}{2}\,
    \norm{\sqrt{g}-\sqrt{y_t}}_2\,\norm{\sqrt{g}-\sqrt{z}}_2, & d=2,\\[1.2em]
    \displaystyle
    \frac{\beta^2-1}{2\beta^2}\,
    \norm{\sqrt{g}-\sqrt{y_t}}_2\,\norm{\sqrt{g}-\sqrt{z}}_2, & d\ge 3.
    \end{cases}
\end{align}
Moreover, for any $y_t\notin B_\eps(z;\Hel)$, we have
\begin{align}
\label{ineq:hellinger-out-of-the-ball-product}
    \inner*{p_{y_t},p_z} \geq - 1.
\end{align}
Overall, combining \eqref{ineq:hellinger-in-the-ball-product} and \eqref{ineq:hellinger-out-of-the-ball-product}, for any $g\notin ( B_{\beta\eps}(z;\Hel)\bigcup \labelspace)$ we have
\begin{align*}
    &\sqrt{2}\inner*{(\nabla F(\sqrt{g}))_{\parallel}, \frac{p_z}{\norm{\sqrt{g}-\sqrt{z}}_2}}\\
    &=\sum_{t=1}^T \frac{\eta_t}{\norm{\sqrt{g}-\sqrt{y_t}}_2} \inner*{p_{y_t},\frac{p_z}{\norm{\sqrt{g}-\sqrt{z}}_2}}\\
    &=\sum_{t:y_t\in B_\eps(z;\Hel)} \eta_t \inner*{\frac{p_{y_t}}{\norm{\sqrt{g}-\sqrt{y_t}}_2},\frac{p_z}{\norm{\sqrt{g}-\sqrt{z}}_2}} + \sum_{t:y_t\notin B_\eps(z;\Hel)} \eta_t \inner*{\frac{p_{y_t}}{\norm{\sqrt{g}-\sqrt{y_t}}_2},\frac{p_z}{\norm{\sqrt{g}-\sqrt{z}}_2}}\\
    &\geq
    \begin{cases}
    \displaystyle
     \eta(B_\eps(z;\Hel)) \cdot \frac{1}{2} - (1 - \eta(B_\eps(z;\Hel))), & d=2,\\[1.2em]
     \displaystyle
     \eta(B_\eps(z;\Hel)) \cdot \frac{\beta^2-1}{2\beta^2} - (1 - \eta(B_\eps(z;\Hel))),  & d\ge 3,
    \end{cases}\\
    &>0.
\end{align*}
where the last inequality is by the assumption that $\eta(B_\eps(z;\Hel)) = \alpha >\alpha_{\Hel}(\beta,d) $. Similar arguments can be made to show that for $g\in \labelspace\setminus B_{\beta\eps}(z;\Hel)$, we have $ t\sqrt{g}\notin \partial F(\sqrt{g})$ for all $t\in \bR$.
These conclude the proof by showing that $\med_\Hel(\labelspace, \eta) \subseteq \set{g: \nabla F(\sqrt{g}) \parallel \sqrt{g} ~\text{or}~ t\sqrt{g}\in \partial F(\sqrt{g}) \text{~for~some~} t\in \bR } \subseteq B_{\beta\eps}(z;\Hel)$.
\end{proof}

The Hellinger result gives a clean concentration--enlargement tradeoff for
density-valued prediction. For KL divergence defined as $\KL(p,q)\ldef \sum_{j=1}^d p_j\log (p_j/q_j)$, however, if one aggregates directly using the geometric median induced by $\sqrt{\KL}$, then
any guarantee with a constant enlargement $\beta$ can require a concentration
level $\alpha$ that is essentially too close to one. The following two-point
example makes this quantitative obstruction visible. We then return to
Hellinger aggregation and use it indirectly to obtain KL exceedance control.

\begin{proposition}[KL impossibility]
\label{prop:kl-impossibility}
For any $\delta\in(0,1/25)$,
letting
$y_1=(\delta,1-\delta)$, $y_2=(1-\delta,\delta),$
and
$
\eta(y_2)=3\sqrt{\delta\log(1/\delta)}$, $\eta(y_1)=1-\eta(y_2)$,
we have
$\med_{\sqrt{\KL}}(\{y_1,y_2\},\eta)\neq\{y_1\}$.
\end{proposition}

The use of $\sqrt{\KL}$ rather than $\KL$ is deliberate. The square-root scale
places KL at the same order as Hellinger distance, and this is the scale on
which median-type stability is meaningful. This mirrors the familiar distinction
between $\ell_2$ and $\ell_2^2$: the squared distance does not have the same
median stability properties as the distance itself.

\begin{proof}[\pfref{prop:kl-impossibility}]
We first verify that the weight assignment is feasible, e.g., $\eta(y_2)\in (0,1)$. For $\delta\in (0, 1/25)$, the function $\delta \to 3\sqrt{\delta\log(1/\delta)}$ is monotonically increasing and we have $3\sqrt{0_+\log(1/0_+)} = 0_+$ and $3\sqrt{\log 25 /25} < 1$. This proves $\eta(y_2)\in (0,1)$.

Define
\[
F(g)\;=\;\eta_1\sqrt{\KL(y_1,g)}+\eta_2\sqrt{\KL(y_2,g)},\qquad g\in\Delta_2.
\]
Since $\KL(y_1,y_1)=0$, we have $F(y_1)=\eta_2\sqrt{\KL(y_2,y_1)}$.
Let $g':=(2\delta,1-2\delta)\in\Delta_2$. We will show $F(g')<F(y_1)$, which proves our claim.

\paragraph{Step 1: bound $\sqrt{\KL(y_1,g')}$.}
A direct calculation yields
\[
\KL(y_1,g')
=
\delta\log\frac{\delta}{2\delta}
+(1-\delta)\log\frac{1-\delta}{1-2\delta}
=
-\delta\log 2+(1-\delta)\log\!\Big(1+\frac{\delta}{1-2\delta}\Big).
\]
Using $\log(1+x)\le x$ and $\delta\le 1/25$, we get
\[
\KL(y_1,g')
\le
-\delta\log 2+\frac{(1-\delta)\delta}{1-2\delta}
\le
\delta\Big(\frac{24}{23}-\log 2\Big)
\le
\frac{2}{5}\,\delta,
\]
hence
\begin{equation}
\sqrt{\KL(y_1,g')}\le \sqrt{\tfrac{2}{5}\delta}.
\label{eq:y1-gprime}
\end{equation}

\paragraph{Step 2: bound $\sqrt{\KL(y_2,g')}$.}
For any $a>b>0$, concavity implies
\[
\sqrt{b}\le \sqrt{a}-\frac{a-b}{2\sqrt{a}}.
\]
Apply this with $a=\KL(y_2,y_1)$ and $b=\KL(y_2,g')$ to obtain
\begin{equation}
\sqrt{\KL(y_2,g')}
\le
\sqrt{\KL(y_2,y_1)}
-\frac{\KL(y_2,y_1)-\KL(y_2,g')}{2\sqrt{\KL(y_2,y_1)}}.
\label{eq:concave-step}
\end{equation}

\paragraph{Step 3: lower bound on $\KL(y_2,y_1)-\KL(y_2,g')$.}
Compute
\begin{align*}
\KL(y_2,y_1)-\KL(y_2,g')
&=
(1-\delta)\log 2
+\delta\log\frac{1-2\delta}{1-\delta}.
\end{align*}
Using $\log(1-x)\ge -x-x^2$ for $x\in[0,1/2]$, we obtain
\[
\delta\log\frac{1-2\delta}{1-\delta}
\ge
-\frac{\delta^2}{1-\delta}-\frac{\delta^3}{(1-\delta)^2}
\ge
-2\delta^2,
\]
where we used $\delta\le 1/25$.
Therefore, for all $\delta\in(0,1/25)$,
\begin{align}
\label{eq:diff-in-kl}
  \KL(y_2,y_1)-\KL(y_2,g')
\ge
(1-\delta)\log 2-2\delta^2
\ge
\Big(\frac{24}{25}-\frac{2}{25^2\log 2}\Big)\log 2
\ge
0.95\,\log 2.
\end{align}

\paragraph{Step 4: upper bound $\KL(y_2,y_1)$.}
Since the density ratio is upper bounded by $1/\delta$,
\begin{align}
\label{eq:density-ratio}
    \KL(y_2,y_1) \leq \log (1/\delta)
\end{align}

\paragraph{Step 5: combine all the bounds.}
Using $\eta_1\le 1$, combining \eqref{eq:y1-gprime} -- \eqref{eq:density-ratio} gives
\begin{align*}
F(g')
&=
\eta_2\sqrt{\KL(y_2,g')}+\eta_1\sqrt{\KL(y_1,g')}
\\
&\le
\eta_2\sqrt{\KL(y_2,y_1)}
-\;\eta_2\cdot \frac{0.95\log 2}{2\sqrt{\KL(y_2,y_1)}}
\;+\;\sqrt{\tfrac{2}{5}\delta}
\\
&\le
F(y_1)
-\;\eta_2\cdot \frac{0.95\log 2}{2\sqrt{\log(1/\delta)}}
\;+\;\sqrt{\tfrac{2}{5}\delta}.
\end{align*}
Substituting $\eta_2=3\sqrt{\delta\log(1/\delta)}$ yields
\begin{align*}
    F(g')
&\le
F(y_1)
-\Big(\frac{3\cdot 0.95\log 2}{2}-\sqrt{\tfrac{2}{5}}\Big)\sqrt{\delta}\\
&\leq F(y_1)-0.35\,\sqrt{\delta}<F(y_1).
\end{align*}
Therefore $y_1$ is not a minimizer of $F$, and consequently
$\med_{\sqrt{\KL}}(\set{y_1,y_2},\eta)\neq\{y_1\}$.
\end{proof}

We can now state the direct KL concentration requirement implied by this construction.

\begin{corollary}[Direct KL stability requires full concentration]
\label{cor:kl-direct-alpha}
For $\beta\ge 1$, let
$\alpha_{\sqrt{\KL},\mathrm{direct}}(\beta)$ denote the smallest concentration
level, with $1$ allowed as the limiting value, at which $\sqrt{\KL}$ on $\Delta_2$ is
$(\alpha,\beta)$-stable by its own geometric median. Then
\[
\alpha_{\sqrt{\KL},\mathrm{direct}}(\beta)=1.
\]
Moreover, even if we only consider finite ensembles $\set{y_i}_{i\in [T]}$ and target $z$ satisfying the
density-ratio bound
\[
c_r\;\ldef\;\sup_{i\in[T],\,j\in[d]}\frac{z_j}{y_{ij}}<\infty,
\]
the direct KL requirement remains large. Namely, for $c_r>24$, let
$\alpha_{\sqrt{\KL},\mathrm{direct}}(\beta,c_r)$ denote the corresponding
smallest concentration level under this ratio constraint. Then
\[
\alpha_{\sqrt{\KL},\mathrm{direct}}(\beta,c_r)
\;\ge\;
\max\left\{
\frac12,\,
1-3\sqrt{\frac{\log(c_r+1)}{c_r+1}}
\right\}.
\]
Thus, even under a density-ratio bound, direct $\sqrt{\KL}$ aggregation can
require concentration close to one when $c_r$ is large.
\end{corollary}

\begin{proof}[\pfref{cor:kl-direct-alpha}]
We first show that no concentration level $\alpha<1$ can guarantee direct
stability without the ratio constraint. Fix any $\alpha\in(\tfrac12,1)$ and
$\beta\ge 1$. Choose
$\delta\in(0,1/25)$ so small that
\[
3\sqrt{\delta\log(1/\delta)}<1-\alpha,
\]
which is possible since $\delta\log(1/\delta)\to0$ as $\delta\downarrow0$.
Let $y_1,y_2$ and $\eta$ be as in \Cref{prop:kl-impossibility}, and set
$z=y_1$ and $\varepsilon=0$. Then
\[
B_0(y_1;\sqrt{\KL})=\{y_1\},
\]
and hence
\[
\eta\!\left(B_0(y_1;\sqrt{\KL})\right)
=
\eta(y_1)
=
1-3\sqrt{\delta\log(1/\delta)}
>
\alpha.
\]
If $\sqrt{\KL}$ were $(\alpha,\beta)$-stable by its own geometric median, this
would imply
\[
\med_{\sqrt{\KL}}(\{y_1,y_2\},\eta)
\subseteq
B_{\beta\cdot 0}(y_1;\sqrt{\KL})
=
\{y_1\}.
\]
Since the geometric median set is nonempty, this would force
$\med_{\sqrt{\KL}}(\{y_1,y_2\},\eta)=\{y_1\}$, contradicting
\Cref{prop:kl-impossibility}. Thus no $\alpha<1$ can work, proving
$\alpha_{\sqrt{\KL},\mathrm{direct}}(\beta)=1$.

For the density-ratio statement, define
\[
a(c_r)\;\ldef\;1-3\sqrt{\frac{\log(c_r+1)}{c_r+1}}.
\]
The lower bound by $1/2$ is built into the definition of the threshold, so it
remains to prove the bound by $a(c_r)$. If $a(c_r)\le 1/2$, there is nothing
to show. Otherwise, fix any $\alpha\in(1/2,a(c_r))$ and set
$\delta=(c_r+1)^{-1}$. Since $c_r>24$, we have $\delta\in(0,1/25)$. Let
$y_1=(\delta,1-\delta)$, $y_2=(1-\delta,\delta)$, set $z=y_1$, and use the
weights from \Cref{prop:kl-impossibility}. The density-ratio constraint holds
because
\[
\sup_{i\in\{1,2\},\,j\in\{1,2\}}\frac{z_j}{y_{ij}}
=
\max\left\{1,\frac{1-\delta}{\delta},\frac{\delta}{1-\delta}\right\}
=
\frac{1-\delta}{\delta}
=
c_r.
\]
Taking $\varepsilon=0$ again gives
\[
\eta\!\left(B_0(z;\sqrt{\KL})\right)
=
\eta(y_1)
=
1-3\sqrt{\delta\log(1/\delta)}
=
a(c_r)
>
\alpha.
\]
If direct $\sqrt{\KL}$ geometric-median aggregation were $(\alpha,\beta)$-stable
under the density-ratio constraint, stability would force
\[
\med_{\sqrt{\KL}}(\{y_1,y_2\},\eta)
\subseteq
B_{\beta\cdot0}(z;\sqrt{\KL})
=
\{y_1\},
\]
again contradicting \Cref{prop:kl-impossibility}. Hence no
$\alpha<a(c_r)$ can guarantee direct stability under the ratio constraint,
which proves the stated lower bound.
\end{proof}

Although direct $\sqrt{\KL}$ aggregation is unstable, the bounded-ratio
condition still permits KL control through Hellinger geometry. When each
candidate density dominates the target up to a factor $c_r$, KL is controlled
by squared Hellinger distance with only logarithmic dependence on $c_r$. We
therefore aggregate with the Hellinger geometric median and then convert the
resulting Hellinger guarantee into a KL guarantee.

\begin{proposition}[KL through Hellinger]
\label{prop:KL-boostability}
For $z\in\Delta_d$ and $\varepsilon>0$, define the KL ball
$B_{\varepsilon}(z;\sqrt{\KL}) \ldef \{y\in\Delta_d : \KL(z, y)\le \varepsilon^2\}$.
Suppose for every $i$ we have the ratio bound $z\le c_r y_i$ coordinate-wise
for some $c_r\ge 1$. Fix $\beta>\sqrt{(2+\log c_r)/3}$.
If $\eta(B_\eps(z;\sqrt{\KL}))>\alpha_{\KL}(\beta,d,c_r)$, where
\[
\alpha_{\KL}(\beta,d,c_r)\;\ldef\;
\begin{cases}
\frac{2}{3}, & d=2,\\[0.5em]
\frac{2\beta^2}{3\beta^2-(2+\log c_r)}, & d\ge 3,
\end{cases}
\]
then every Hellinger geometric median of $(\labelspace,\eta)$ lies in the enlarged KL ball:
\[
\med_{\Hel}(\labelspace,\eta)\subseteq B_{\beta\varepsilon}(z;\sqrt{\KL}).
\]
\end{proposition}

Numerically, the two KL routes have very different concentration requirements.
The direct $\sqrt{\KL}$ median can require 
$
\alpha_{\sqrt{\KL},\mathrm{direct}}(\beta,c_r)
\ge
1-3\sqrt{\frac{\log(c_r+1)}{c_r+1}},
$
so for large $c_r$ the weak guarantee must concentrate almost all aggregation
weight near the target. In contrast, the Hellinger route keeps the required
concentration at a constant level if the KL-radius enlargement is allowed to
scale logarithmically. For example, when $d\ge3$ and
$\beta^2=\lambda(2+\log c_r)$ with $\lambda>1$,
\[
\alpha_{\KL}(\beta,d,c_r)=\frac{2\lambda}{3\lambda-1}.
\] 
Taking $\lambda=2$ gives $4/5$, and
$\lambda=3$ gives $3/4$.

\begin{lemma}[Lemma 4 of \citep{yang1998asymptotic}]
\label{lem:relation-kl-hels}
For any $p,q\in \bR_{>0}^d$, if $p\le c_r q$ coordinate-wise for some
$c_r\ge 1$, then
\[
    \Hels(p,q) \leq \GKL(p, q) \leq \prn*{2+\log c_r}\cdot\Hels(p,q).
\]
\end{lemma}

\begin{lemma}
\label{lem:med-preserve-ratio-cover}
Let $z, y_1,\dots,y_T\in \Delta^d$. Suppose for every $t$, $z\le c_r y_t$
coordinate-wise for some $c_r\ge 1$. Then for any
$g\in \med_\Hel(\labelspace,\eta)$, we have $z\leq c_r g$.
\end{lemma}

\begin{proof}[\pfref{lem:med-preserve-ratio-cover}]
We first define the projected spherical convex hull of the set
$\labelspace=\set{y_1,\dots,y_T}$ by
\[
    \convp(\labelspace) \ldef \set*{y =
    \frac{\prn*{\sum_t a_t\sqrt{y_t}}^2}
    {\norm{\sum_t a_t\sqrt{y_t}}_2^2} \mid{} a\in \Delta^T},
\]
where $\sqrt{\cdot}$ and $(\cdot)^2$ are both element-wise. For any
$y\in \convp(\labelspace)$, the ratio assumption gives
\[
    \sqrt{z}\leq  \sum_t a_t\sqrt{c_r y_t}.
\]
Since $\norm{\sum_t a_t\sqrt{y_t}}_2 \leq \sum_t a_t \norm{\sqrt{y_t}}_2\leq 1$,
we have for all $y\in \convp(\labelspace)$,
\begin{align}
\label{ineq:convp-is-good-enough}
    z\leq \prn*{\sum_t a_t\sqrt{c_r y_t}}^2
    \leq c_r\frac{\prn*{\sum_t a_t\sqrt{y_t}}^2}
    {\norm{\sum_t a_t\sqrt{y_t}}_2^2}= c_r y.
\end{align}
It remains to show that $g\in \convp(\labelspace)$. We prove this by
contradiction. Assume $g\notin \convp(\labelspace)$, and let
$y_g = \argmin_{y\in\convp(\labelspace)} \norm{\sqrt{y}-\sqrt{g}}_2$.
Since $\convp(\labelspace)$ is spherically convex, the first-order condition
for projecting $\sqrt g$ onto $\convp(\labelspace)$ implies, for every $t$,
\begin{align}
\label{ineq:i}
    \inner{\sqrt{y_t}, \sqrt{g}}\leq
    \inner{\sqrt{y_t}, \sqrt{y_g}}\inner{\sqrt{g},\sqrt{y_g}}
    \leq \inner{\sqrt{y_t}, \sqrt{y_g}}.
\end{align}
Thus
\[
    \nrm{\sqrt{g} -\sqrt{y_t}}_2^2
    \geq
    \nrm{\sqrt{y_g} -\sqrt{y_t}}_2^2
    \qquad\text{for all }t.
\]
By the optimality of $g\in\med_\Hel(\labelspace,\eta)$, equality must hold for
all $t$. If $\inner{\sqrt{y_t},\sqrt{y_g}}>0$ for some $t$, equality in
\eqref{ineq:i} forces $\inner{\sqrt g,\sqrt{y_g}}=1$, hence $g=y_g\in
\convp(\labelspace)$, a contradiction. Otherwise
$\inner{\sqrt{y_t},\sqrt{y_g}}=0$ for all $t$, and then
$\inner{\sqrt{y_t},\sqrt g}=0$ for all $t$ as well. In this case
$\Hel(y_t,g)=1\ge \Hel(y_t,x)$ for every $x\in\Delta^d$, and $y_1$ gives a
strictly smaller weighted Hellinger objective:
\[
    \sum_t \eta_t \nrm{\sqrt{y_1}-\sqrt{y_t}}_2
    =
    \sum_{t=2}^T \eta_t \nrm{\sqrt{y_1}-\sqrt{y_t}}_2
    <
    \sum_{t=1}^T \eta_t \nrm{\sqrt g-\sqrt{y_t}}_2.
\]
This again contradicts $g\in\med_\Hel(\labelspace,\eta)$.

Therefore $g\in \convp(\labelspace)$, and \eqref{ineq:convp-is-good-enough}
implies $z\le c_r g$.
\end{proof}

\begin{proof}[\pfref{prop:KL-boostability}]
By \Cref{lem:relation-kl-hels}, $\Hels\le\KL$, hence
\[
\eta(B_\varepsilon(z;\Hel))
\ge
\eta(B_\varepsilon(z;\sqrt{\KL}))
>
\alpha_{\KL}(\beta,d,c_r).
\]
With $\beta'=\beta/\sqrt{2+\log c_r}$, the threshold matches the Hellinger
stability threshold in \Cref{prop:hel-boostability}, implying
$\med_{\Hel}(\labelspace,\eta)\subseteq B_{\beta'\varepsilon}(z;\Hel)$.
The ratio-cover assumption and \Cref{lem:med-preserve-ratio-cover} then give
$\KL(z,g)\le (2+\log c_r)\Hels(z,g)\le (\beta\varepsilon)^2$ for any
$g\in\med_{\Hel}(\labelspace,\eta)$.
\end{proof}

\begin{corollary}[KL exceedance via Hellinger aggregation]
\label{cor:kl-via-hel-aggregation}
Suppose $\alpha-\gamma>\alpha_{\KL}(\beta,d,c_r)$ for some
$\beta>\sqrt{(2+\log c_r)/3}$, $\WL$ is a weak learner with parameter
$\alpha$, and the ratio condition $y_i\le c_r h_t(x_i)$ holds coordinate-wise
for every sample $i$ and round $t$. Let $f_T$ be the output of \geomedboost\
whose final aggregation step uses the Hellinger geometric median,
\[
f_T(x)\in \med\bigl(\{h_t(x)\}_{t=1}^T,\set{\eta_t'}_{t=1}^T;\Hel,\gamma\bigr),
\]
and whose exceedance loss is evaluated using $\sqrt{\KL}$. Then
\[
L_{\sqrt{\KL},\beta\varepsilon}(f_T)
\;\le\;
\prod_{t=1}^{T}
\sum_{i=1}^n
w_t(x_i)\exp\!\Bigl(-\eta_t\bigl(1-\alpha-C_{\varepsilon}(y_i,h_t(x_i))\bigr)\Bigr).
\]
Moreover, take
$C_\varepsilon(y,y')=\indic\{\sqrt{\KL(y,y')}>\varepsilon\}$. If, in
addition, for some $\zeta>0$ with $\alpha+\zeta\leq 1$ the same weak learner
has parameter $\alpha+\zeta$, then
\[
L_{\sqrt{\KL},\beta\varepsilon}(f_T)\leq \exp(-2\zeta^2T).
\]
Consequently, the empirical exceedance loss becomes zero once
$T>\log(n)/(2\zeta^2)$.
\end{corollary}

\section{Conclusion}
\label{sec:conclusion}

We introduced $(\alpha,\beta)$-stability by geometric median as the geometric
property of a target divergence that explains when aggregation can turn weak guarantees into strong
ones for vector-valued prediction and conditional density estimation. 
With this property, our boosting algorithm \geomedboost achieves exponential decay
of the empirical exceedance loss, and the empirical guarantee then transfers to population error bounds through the generalization analysis.
We also characterized this stability property for the natural divergences
studied in the paper. For unconstrained vector-valued prediction, we found sharp $(\alpha,\beta)$-stability for $\ell_1$ and $\ell_2$ respectively. For conditional
density estimation, we identify similar $(\alpha,\beta)$-stability for $\TV$ and $\Hel$. The KL divergence, however,
requires an indirect route through Hellinger aggregation.
Together, these findings takes a steady step towards clarifying both the
statistical and geometric mechanisms underlying boosting beyond scalar prediction.

{
\bibliography{refs}
}

\appendix
\section{Technical Lemma}

\label{app:technical-lemmas}

\begin{lemma}[Hoeffding bound for one round]
\label{lem:hoeffding-round-factor}
Let $Z$ be a random variable with range contained in an interval of length
$M>0$ and $\En Z\geq \zeta>0$. Then
\[
\inf_{\eta>0}\En e^{-\eta Z}\leq \exp\!\left(-\frac{2\zeta^2}{M^2}\right).
\]
\end{lemma}

\begin{proof}[\pfref{lem:hoeffding-round-factor}]
Hoeffding's lemma gives, for every $\eta>0$,
\[
\En e^{-\eta Z}
\leq
\exp\!\left(-\eta \En Z+\frac{\eta^2M^2}{8}\right)
\leq
\exp\!\left(-\eta\zeta+\frac{\eta^2M^2}{8}\right).
\]
Taking infimum on the left hand side and $\eta=4\zeta/M^2$ on the right hand side gives the claim.
\end{proof}

\end{document}